\documentclass{article}
\usepackage{ijcai21}

\usepackage{times}
\usepackage{soul}
\usepackage{url}
\usepackage[hidelinks]{hyperref}

\usepackage[utf8]{inputenc}
\usepackage[small]{caption}
\usepackage{graphicx}
\usepackage{amsmath}
\usepackage{amsfonts}
\usepackage{amsthm}
\usepackage{booktabs}
\usepackage{algorithm}
\usepackage{algorithmic}
\usepackage{subfigure}
\usepackage{svg}
\usepackage[inline]{enumitem}
\newtheorem{theorem}{Theorem}

\newcommand{\cmmnt}[1]{}
\title{Probability Density Estimation Based Imitation Learning }

\author{
    Yang Liu$^1$ \and
    Yongzhe Chang$^1$ \and 
    Shilei Jiang$^1$\and
    Xueqian Wang$^1$\and
    Bin Liang$^{1,2}$\and
    Bo Yuan$^1$\footnote{Contact Author}
  \affiliations
  $^1$Shenzhen International Graduate School, Tsinghua University\\
  $^2$Department of Automation, Tsinghua University\\
  \emails
    liuyang-19@mails.tsinghua.edu.cn, 
    changyongzhe@sz.tsinghua.edu.cn,
    jsl19@mails.tsinghua.edu.cn,
    wang.xq@sz.tsinghua.edu.cn,
    liangbin@tsinghua.edu.cn, 
    yuanb@sz.tsinghua.edu.cn
}
\begin{document}

\maketitle

\begin{abstract}
  Imitation Learning  (IL) is an effective learning paradigm exploiting the interactions between agents and environments. 
  It does not require explicit reward signals and instead tries to recover desired policies using expert demonstrations. 
  In general, IL methods can be categorized into Behavioral Cloning  (BC) and Inverse Reinforcement Learning  (IRL).
  In this work, a novel reward function based on probability density estimation is proposed for IRL, which can significantly reduce the complexity of existing IRL methods. 
  Furthermore, we prove that the theoretically optimal policy derived from our reward function is identical to the expert policy as long as it is deterministic. 
  Consequently, an IRL problem can be gracefully transformed into a probability density estimation problem. 
  Based on the proposed reward function, we present a “watch-try-learn” style framework named 
  Probability Density Estimation based Imitation Learning  (PDEIL), which can work in both discrete and continuous action spaces. 
  Finally, comprehensive experiments in the Gym environment show that PDEIL is much more efficient than existing algorithms in recovering rewards close to the ground truth.
\end{abstract}

\section{Introduction}

Recently, Reinforcement Learning  (RL) has made remarkable success in many tasks, such as Go  \cite{silver2017mastering},
Atari games  \cite{mnih2013playing} and robot control  \cite{abbeel2004apprenticeship,andrychowicz2020learning}. In an RL system, 
in addition to the agent and the environment, there are 
four key elements: policy, reward signal,  value function, and  model  \cite{sutton1998introduction}, among which
the reward signal plays a key role in directly determining the overall performance of the RL system.
More specifically, the reward signal defines the goal of the RL task, and guides the learning agent towards the desirable direction.
However, in some real-world problems, the reward signal is difficult to receive or to measure  explicitly.
In such cases, alternative approaches such as  Imitation Learning (IL) are required.

In IL, the learning agent receives demonstrations from an expert,
and the goal is to recover a desired policy using the expert demonstrations through interactions with the environment.
In general, IL methods can be categorized as Behavioral Cloning  (BC) \cite{pomerleau1991efficient,ross2010efficient,bagnell2007boosting} 
and Inverse Reinforcement Learning  (IRL)  \cite{russell1998learning,ng2000algorithms}. 
In BC, the agent tries to mimic the expert's action in each state, with the aim to match the expert as closely as possible.
By contrast, IRL methods  try to recover a reward function based on the demonstrations of the expert, 
and then use RL algorithms to train an agent based on the recovered reward function. In practice, BC methods are usually used to derive an initial policy for RL algorithms  
\cite{nagabandi2018neural,rajeswaran2017learning},
while IRL methods employ an iterative process alternating between reward recovery and RL.

However, BC methods have two major flaws. 
First, the prediction error may accumulate and agents in BC cannot recover the true action once mistakes occur.   
Second, BC methods are  supervised learning methods and require a large amount of data for training.
Furthermore, supervised learning methods assume that
the demonstrations are \textit{i.i.d}, which is unlikely to be the case in real-world tasks. 

Most previous IRL methods are model-based and need to call a Markov decision process (MDP) solver multiple times during learning. 
 Meanwhile, IRL methods are essentially cost function learning methods \cite{ng2000algorithms,abbeel2004apprenticeship,ziebart2008maximum}, aiming to minimize
the distance between the expert and the learning agent. For example,
Adversarial IL \cite{ho2016generative} uses  the KL divergence to estimate the cost function,  while Primal Wasserstein Imitation Learning  (PWIL)   \cite{dadashi2020primal} uses the upper bound of the Wasserstein distance to estimate the cost, 
and Random Expert Distillation  (RED)  \cite{wang2019random} uses the expert policy support estimation to estimate the cost.

In this paper, we focus on how to recover the reward function in IRL. The main contribution is 
that we propose a concise method to recover the reward function using expert demonstrations
based on probability density estimation. 
Instead of being restricted by the concept of cost, we propose to recover the reward function directly, and the form of the reward function is also
simpler and easier to calculate than previous IRL methods. We prove that, in theory, the expert policy can be recovered when our reward function is adopted.
We also propose a novel IRL algorithm named 
Probability Density Estimation based Imitation Learning  (PDEIL) to demonstrate the 
efficacy of the proposed reward function. 
Experimental studies in both discrete and continuous action spaces confirm that 
the desired policy can be recovered requiring fewer interactions with the environment and 
less computing resources compared with existing IL methods. 

\section{Background and Notations}

An environment in RL can be modeled as an MDP  \cite{sutton1998introduction}, 
with a tuple $ (S, A, T, \gamma, D, R)$, where $S$ is the state space; $A$ is
the action space; $T$ is the state transition probability model; 
$\gamma \in  (0, 1)$ is the discount factor; $D$ is the initial state distribution from which $s_0$ is drawn; 
$R: S \times A \rightarrow \mathbb{R}$ is the reward signal. A policy is 
the rule of the agent's behavior that can be generally denoted as $\pi$. 
The goal of RL is to obtain a policy that maximizes the cumulative discounted reward from the initial state, 
denoted by the performance objective $J (\pi)$, which can be written as an expectation \cite{silver2014deterministic}:
\begin{equation}
  \begin{aligned}
    J\left (\pi\right) &=\mathbb{E}_{s \sim \rho^{\pi}, a \sim \pi_{\theta}}[r (s, a)] \\
    &=\int_{S} \rho^{\pi} (s) \int_{A} \pi (a|s) r (s, a) \mathrm{d} a \mathrm{d} s 
  \end{aligned}
  \label{optimal_objective}
\end{equation}
where $\rho^{\pi} (s) = \sum_{t = 0}^{\infty}\gamma^{t}p (s_t=s|\pi)$ is the discounted state 
distribution of policy $\pi$. In the following, we  
introduce a new discounted distribution $\rho^{\pi} (s,a) = \rho^{\pi} (s) \pi (a|s)$, which is called discounted
state-action joint distribution of policy $\pi$. To differentiate between the two discounted distributions, 
we denote the discounted state distribution as $\rho^{\pi}_{s} (\cdot)$, and the discounted state-action joint
distribution as $\rho^{\pi}_{s,a} (\cdot)$.

Since there is no reward signal in the environment of IL, we use MDP$\backslash$R to denote an MDP without reward signal \cite{abbeel2004apprenticeship}, 
in the form of a tuple $ (S, A, T, \gamma, D)$.
Meanwhile, we use $\pi_{e}$ to represent the expert policy from which the demonstrations are generated.

\section{Related Works}

IRL was first defined by Ng \cite{ng2000algorithms}, which aimed to optimize the reward function given expert demonstrations and access to 
the environment. IRL itself features inherent ambiguity as the same optimal policy can be potentially derived with different reward functions.
Some IRL methods add an extra constraint to deal with the ambiguity issue, such as the maximum entropy  principle \cite{ziebart2008maximum}. 
Traditional IRL methods often employ a linear combination of features as the reward function and the task of recovering the reward function is to 
optimize a set of weights \cite{abbeel2004apprenticeship}.

Adversarial IL is an IRL method that has recently attracted significant interests, 
which uses Generative Adversarial Network (GAN) \cite{goodfellow2014generative} to simulate the IL process. 
In Adversarial IL, the generator in GAN represents a policy, while the discriminator in GAN corresponds to the reward function, 
and the recovering of the reward function and the learning of policy in IL are transformed into the training of the discriminator and generator. 

Expert support estimation is another direction in IRL. 
The key idea is to encourage the agent to stay within the state-action support of the expert 
and several reward functions have been proposed, such as Soft Q Imitation Learning (SQIL)  
\cite{reddy2019sqil}, RED, Disagreement-Regularized Imitation Learning (DRIL) \cite{brantley2019disagreement}, and PWIL.
SQIL features a binary reward function, according to which when the agent selects an action in a state endorsed by the expert, 
the agent receive a reward of $+1$.
Otherwise, the agent receive a reward of $0$. RED uses a neural network to estimate  the state-action support of the expert.
DRIL relies on the variance among an ensemble of BC models to estimate the state-action support, and constructs a reward function based on the distance to the 
support and the KL divergence among the BC models. 
PWIL employs a “pop-outs” trick to enforce agent to stay within the expert's support. 
In some sense, our method is similar to these methods as it employs a probability model to estimate the expert support as a component of the reward function.
% 还需要再一点点点“”

\section{Methodology}

The objective of our work is to design a reward function that can make the resulting optimal policy equal to the expert policy.
In this section, we first present the structure of our reward function, 
and then investigate why our reward function can make the optimal policy equal to the expert policy. 
Furthermore, we propose a revised version of the original reward function that can overcome its potential issue.
Finally, the PDEIL algorithm that employs the reward function is introduced.

\subsection{Reward Based on  Probability Density Estimation}

In traditional RL, 
the goal of the agent is to seek a policy that maximizes $J (\pi)$ in Equation \eqref{optimal_objective}. 
However, in an MDP$\backslash$R environment, 
the agent cannot receive reward signals from the interactions with the environment.
To accommodate this lack of reward signals, 
we need to design a reward function based on expert demonstrations.
If the reward function can guarantee in theory that the corresponding optimal policy is equal to the expert policy, 
we can expect to recover the expert policy using an RL algorithm.

\begin{theorem}\label{as1}
  Assume $\pi_e$ is a deterministic policy, then:
    \begin{equation}
      \forall \pi, \langle \pi_e (a | s), \pi_e (a | s) \rangle \geq  \langle \pi (a | s), \pi (a | s) \rangle,
      \label{ase1}
    \end{equation}
\end{theorem}
\noindent where $\langle  \cdot, \cdot \rangle$ is the inner product; $\langle \pi (a | s), \pi (a | s) \rangle = \int_a\pi^2 (a|s)da$ for continuous action spaces and 
$\langle \pi (a | s), \pi (a | s) \rangle =  \sum_a\pi^2 (a|s)$ for discrete action spaces. 
 Moreover, if a policy $\pi$ becomes more deterministic, the value of $\langle \pi (a | s), \pi (a | s) \rangle$ will also increase, and 
$\langle \pi (a | s), \pi (a | s) \rangle$ can be used to measure  the stochasticity of a policy 
(i.e., similar to the entropy of a policy \cite{haarnoja2018soft}).

\begin{proof}
  When the action space is discrete, and $\pi_e$ is a deterministic policy, 
  the expert agent only selects 
  an action with probability 1 for each state, then:
  $$\langle \pi_e, \pi_e \rangle = 1,$$
  for all policies:
  \begin{equation*}
      \langle \pi, \pi \rangle = \sum_{a \in A}\pi^{2} (a|s),
  \end{equation*}
  for all actions:  
  $$0 \leq \pi (a | s) \leq 1,$$
  then:
  $$\langle \pi, \pi \rangle \leq \sum_{a \in A}\pi (a|s) = 1,$$
  we have:
  $$\forall \pi, \langle \pi_e, \pi_e \rangle \geq \langle \pi, \pi \rangle$$
  % 这里这段话需要一定的修改也需要添加引用
  When the action space is continuous, the probability density function of a deterministic policy is a shifted Dirac function  ($\delta (a - a_0)$) \cite{lillicrap2015continuous}. Note that $\int_a\pi_e^2 (a|s)da$ is not 
  integrable when $\pi_e (a | s) = \delta (a - a_0)$, and $ \langle \pi_e, \pi_e \rangle$ goes to infinity, so intuitively, $\langle \pi_e, \pi_e \rangle \geq \langle \pi, \pi \rangle$ .  % 这句话需要填写一些新的 
\end{proof}

Table \ref{t1} is an example of Theorem \ref{as1}  when the action space is discrete. In table {\ref{t1}}, $\langle \pi_1 (a | s), \pi_1 (a | s) \rangle = 1, \langle \pi_2 (a | s), \pi_2 (a | s) \rangle = \frac{1}{3}$. Since $\pi_1$ is a 
deterministic policy  (it only selects action $a_3$), and $\pi_2$ is a uniform random policy,
therefore $\langle \pi_1 (a | s), \pi_1 (a | s) \rangle \geq \langle \pi_2 (a | s), \pi_2 (a | s) \rangle$.

  \begin{table}[tbh]
    \centering
    \caption{\small{An example of two different policies on the same state with a discrete action space.}}
    \begin{tabular}{llll}
      \toprule
               & $a_1$ & $a_2$ & $a_3$ \\ \midrule
    $\pi_1 (a|s)$ & $0$    & $0$  & $1$  \\
    $\pi_2 (a|s)$ & $\frac{1}{3}$  & $\frac{1}{3}$  & $\frac{1}{3}$  \\ \bottomrule
  \end{tabular}
  \label{t1}  
\end{table}

Figure \ref{kk}  gives an example of  Theorem \ref{as1}  when the action space is continuous. It shows two alternatives to approach the shifted  Dirac function: 
uniform distributions as Figure \ref{un}, and triangle distributions as Figure \ref{tri}. 
When the height of the rectangle or triangle goes to infinity, the probability density function approaches $\delta (a - a_0)$.
Meanwhile, the higher the rectangle or triangle, the more deterministic the policy is. If we calculate the inner product of the policies in Figure\ref{kk}, 
it is clear that more deterministic policies have greater inner product values.
\begin{figure}[tbh]
  \centering
  \subfigure[Uniform policies]{
    \resizebox{0.2\textwidth}{!}{\includesvg{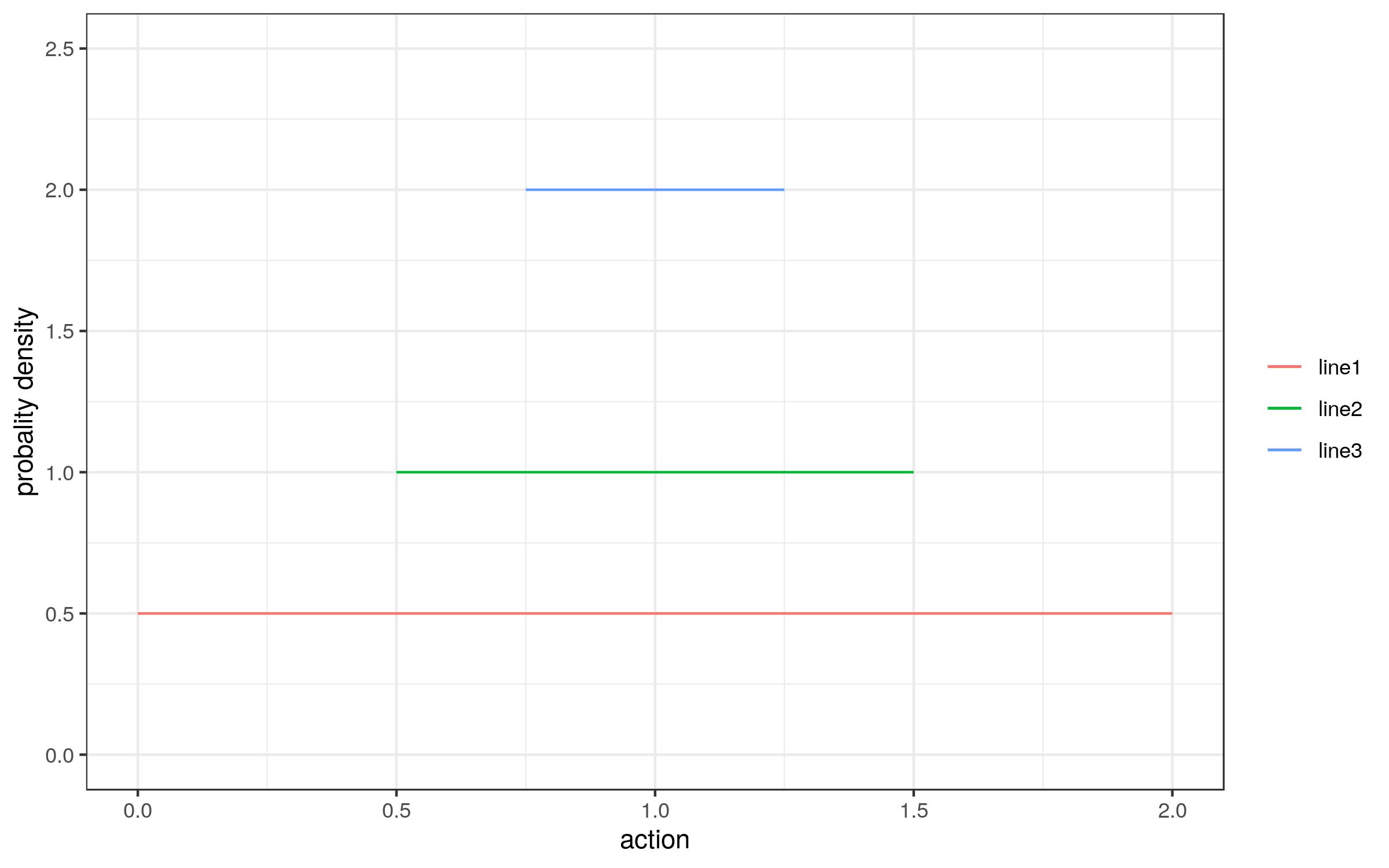}}
    \label{un}
  }
  \subfigure[Triangle policies]{
    \resizebox{0.2\textwidth}{!}{\includesvg{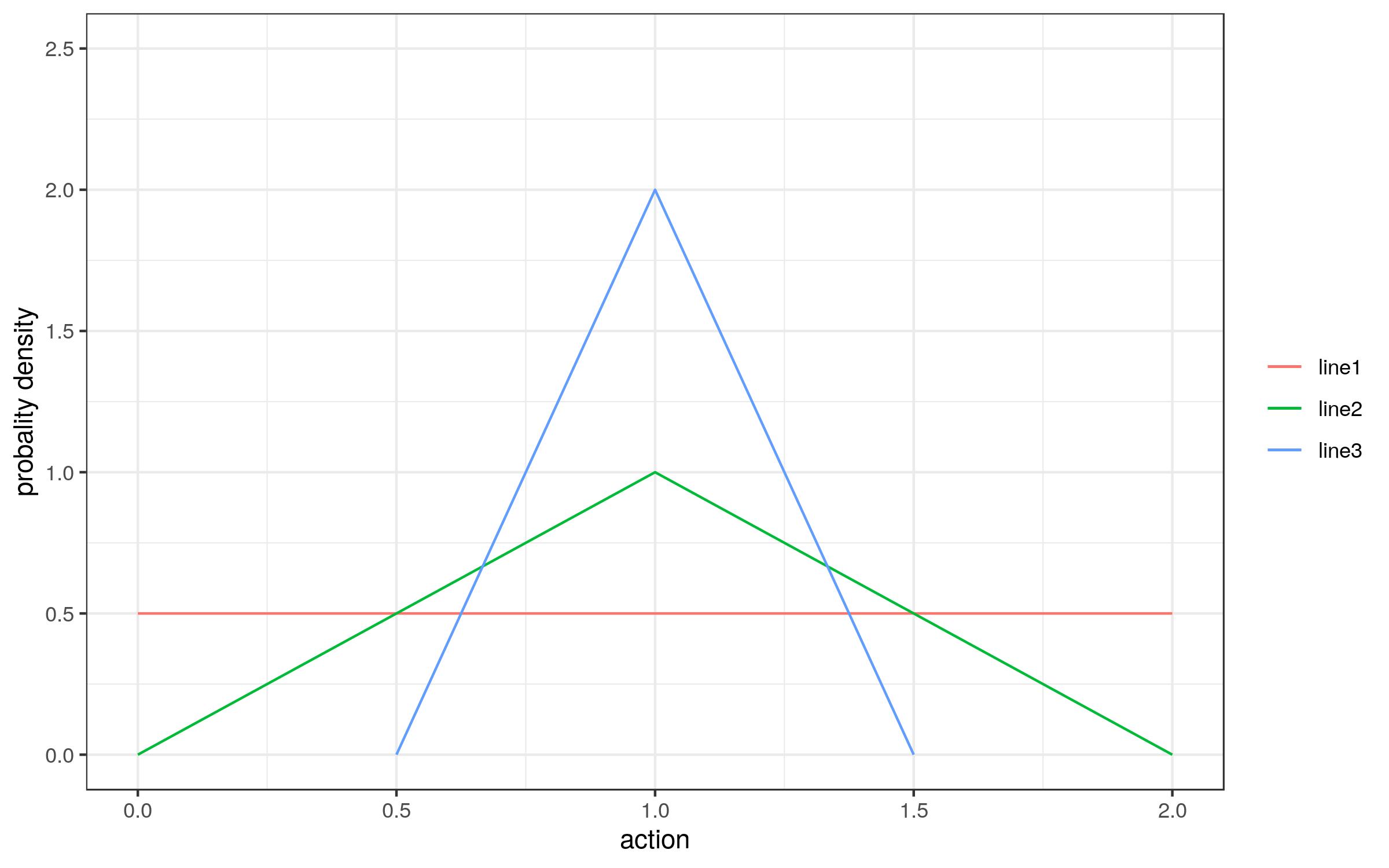}}
    \label{tri}
  }
  \caption{An example of different policies on the same state with a continuous action space. In uniform policies, $\langle \pi_1, \pi_1 \rangle = 0.5$, $\langle \pi_2, \pi_2 \rangle = 1$, $\langle \pi_3, \pi_3 \rangle = 2$, 
  $\langle \pi_1, \pi_1 \rangle < \langle \pi_2, \pi_2 \rangle < \langle \pi_3, \pi_3 \rangle$; In triangle policies $\langle \pi_1, \pi_1 \rangle = 0.5$, 
  $\langle \pi_4, \pi_4 \rangle = \frac{2}{3}, \langle \pi_5, \pi_5 \rangle = \frac{4}{3}$, 
  $\langle \pi_1, \pi_1 \rangle < \langle \pi_4, \pi_4 \rangle < \langle \pi_5, \pi_5 \rangle$.}
  \label{kk}
\end{figure}

\begin{theorem}\label{pro}
    Assume:
    \begin{equation}
      \begin{aligned}
        r (s, a) &= \frac{\rho_{s}^{\pi_e} (s)}{\rho^{\pi}_{s} (s)}\pi_{e} (a | s) \\
        &= \frac{\rho_{s, a}^{\pi_e} (s,a)}{\rho_{s}^{\pi} (s)}
      \end{aligned},
      \label{origin_reward}  
    \end{equation}
    when the expert policy is a deterministic policy (i.e., Equation \eqref{ase1} is satisfied), the optimal policy $\pi_{*}$ is identical to the expert policy 
    under the optimal objective of Equation \eqref{optimal_objective}.
\end{theorem}
  
  \begin{proof}
    Since Equation \eqref{origin_reward}, then:
    \begin{equation*}
      \begin{aligned}
        J (\pi) &= \int_{S} \rho^{\pi}_{s} (s) \int_{A} \pi (a|s) r (s, a) dads \\
               &= \int_{S}\rho^{\pi}_{s} (s) \int_{A} \pi (a|s)\frac{\rho_{s}^{\pi_e} (s)}{\rho^{\pi}_{s} (s)}\pi_{e} (a | s)dads \\
               &= \int_{S}\rho^{\pi_e}_{s} (s) \int_{A} \pi (a|s)\pi_{e} (a | s)dads 
      \end{aligned}
    \end{equation*}
    \begin{equation*}
      \begin{aligned}
        J (\pi_e) &= \int_{S} \rho^{\pi_e}_{s} (s) \int_{A} \pi_e (a|s) r (s, a) dads \\
                 &= \int_{S}\rho^{\pi_e}_{s} (s) \int_{A} \pi_e (a|s)\frac{\rho_{s}^{\pi_e} (s)}{\rho^{\pi_e}_{s} (s)}\pi_{e} (a | s)dads \\
                 &= \int_{S}\rho^{\pi_e}_{s} (s) \int_{A} \pi_e (a|s)\pi_{e} (a | s)dads \\ 
      \end{aligned}
    \end{equation*}
    Given the fact of Cauchy Inequality:
    \begin{equation*}
       (\int f (x) g (x) dx)^2 \leq \int f^2 (x)dx \int g^2 (x)dx
    \end{equation*}
    then:
    \begin{equation*}
       (\int_{A} \pi (a|s)\pi_{e} (a | s))^2da \leq \int_{A}\pi^2 (a|s)da \int_{A} \pi_e^2 (a|s)da  
    \end{equation*} 
    $\exists$ Equation \eqref{ase1}, then:
    \begin{equation*}
      \begin{aligned}
         (\int_{A} \pi (a|s)\pi_{e} (a | s))^2da &\leq \int_{A}\pi^2 (a|s)da \int_{A} \pi_e^2 (a|s)da \\
                                             &\leq  (\int_{A} \pi_e^2 (a|s)da)^2
      \end{aligned}
    \end{equation*}
    then:
    \begin{equation*}
      J (\pi) \leq J (\pi_e), \pi_* = \pi_{e}
    \end{equation*}
  \end{proof}

According to Theorem \ref{pro}, we can construct the reward function 
as Equation  \eqref{origin_reward}, 
and once the optimal policy based on this reward function is found, 
we can recover the expert policy. 
Since the distribution of population is not known \textit{a priori}, $\rho^{\pi_e}_{s,a} (s,a)$ and $\rho^{\pi}_{s} (s)$ cannot be computed directly. 
The most common and intuitive solution is to estimate
the two probability densities from corresponding samples. Consequently, the practical reward function can be written as:
  \begin{equation}
    r (s, a) = \frac{\widehat{\rho^{\pi_e}_{s, a}} (s,a)}{\widehat{\rho^{\pi}_{s}} (s)},
    \label{mis_lead}
  \end{equation}
where $\widehat{\rho^{\pi_e}_{s, a}} (s,a)$  can be estimated using the demonstrations of the expert,
and $\widehat{\rho^{\pi}_{s}} (s)$ can be estimated through the agent's interactions with the environment. 
However, the reward function in the original form of Equation \eqref{mis_lead} has a major defect, which we will discuss in the following section.

  \subsection{Misleading Reward}\label{mr}
  
  \begin{figure}[tb]
    \centering
    \subfigure[Change in ratio]{
      \resizebox{0.45\textwidth}{!}{\includesvg{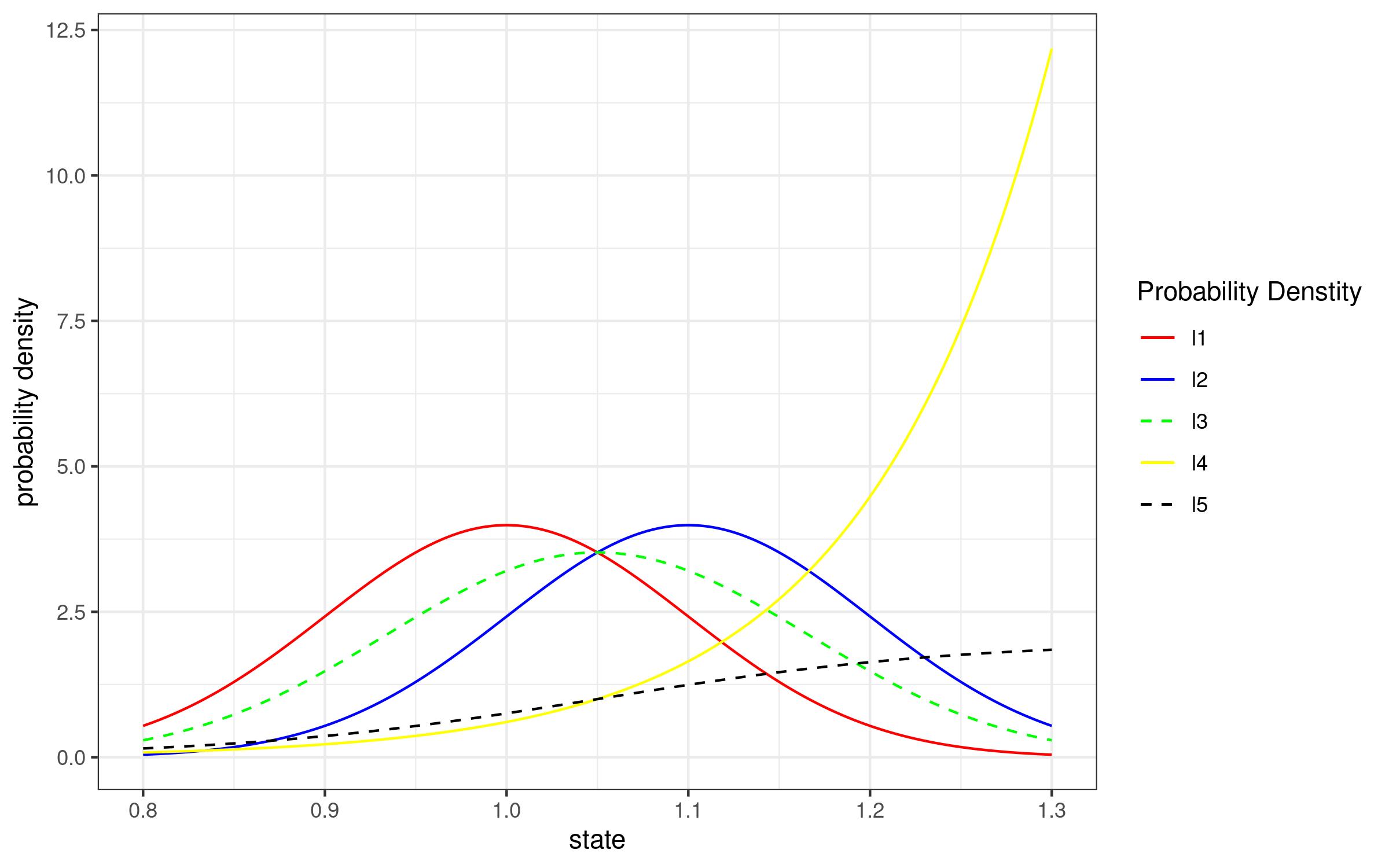}}
    }
    \subfigure[Reward variance]{
      \resizebox{0.45\textwidth}{!}{\includesvg{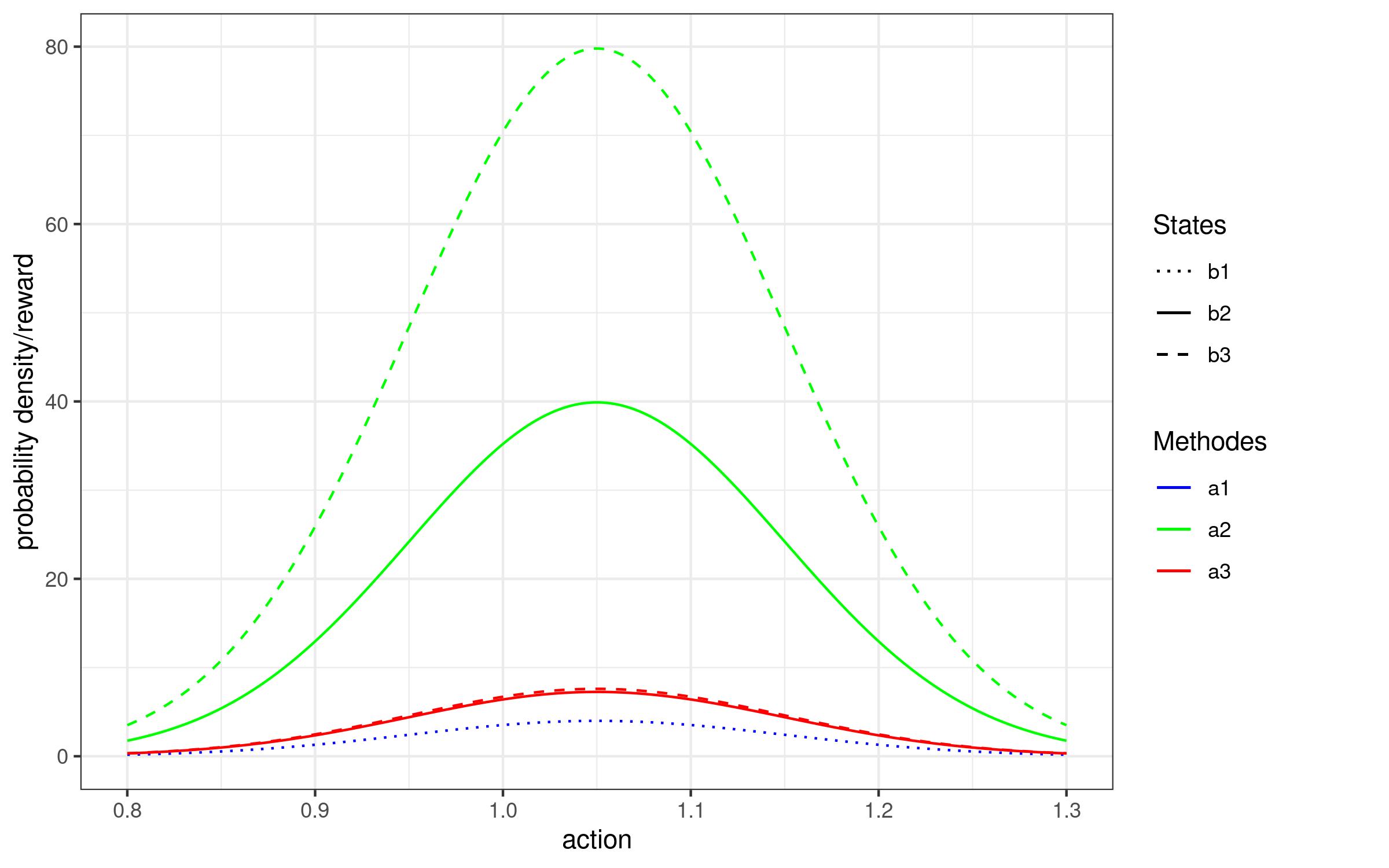}}
      \label{haha}
    }
    \caption{Comparison between two reward functions. 
    The original ratio $ \frac{\rho_{s}^{\pi_e} (s)}{\rho^{\pi}_{s} (s)}$ may be extremely high when $\rho^{\pi}_{s} (s)$ is close to 0 
    , while the revised ratio $\frac{2\rho_{s}^{\pi_e} (s)}{\rho_{s}^{\pi_e} (s) + \rho^{\pi}_{s} (s)}$ is close the original ratio when the original ratio is low, 
    and has an upper bound of 2.
    The reward in the form of Equation \eqref{mis_lead} has a higher variance than the revised reward in the form of Equation \eqref{adjust_error}.}
    \label{vbt}
  \end{figure}

  The reward function in the form of Equation \eqref{mis_lead} has a potential issue, referred to as the misleading reward problem.
  It means that the agent relying on this reward function may get extremely wrong rewards in some cases. 
  This problem happens when the agent reaches a state that it never encountered before.
  When the agent explores such states, where the values of $\widehat{\rho^{\pi}_{s}} (s)$ are close to zero, 
  the estimated rewards in these states are likely to have a large 
  variance among different actions. Consequently, the agent will receive very high reward signals in these states compared with other ordinary states. 
  These wrong rewards can inevitably mislead the agent, and make the RL algorithm fail to reach its intended target. % 插入一张 misleading reward 的图%
  
  The cause of this misleading reward problem is the estimation error of $\rho^{\pi}_{s}$. 
  When we use certain states as the samples to estimate the probability density, 
  the probability density of other states may be estimated to be close to 0. 
  Although this problem can be partially alleviated by increasing the number of states used to estimate $\rho^{\pi}_{s}$, 
  more interactions will also be required, which is not sample-efficient.

  To solve this problem, we make a trade-off between bias and variance in estimating the reward:

  \begin{equation}
    r (s, a) = \frac{\widehat{\rho^{\pi_e}_{s,a}} (s,a)}{\alpha \widehat{\rho^{\pi_e}_{s}} (s)+ \beta \widehat{\rho^{\pi}_{s}} (s)},
    \label{origin_adjust_error}
  \end{equation}
  where $\alpha + \beta = 1$ and $0 \leq \alpha \leq 1$. The coefficient $\alpha$ plays the role of a variance controller: when $\alpha$ is close to $0$, 
  the estimator has high variance and low bias; when $\alpha$ is close to 1, the estimator has high bias and low variance. 
  
  Intuitively, $\alpha = 0.5$ indicates a reasonable balance between bias and variance, and a revised reward function is:
  \begin{equation}
    r (s, a) = \frac{2 \widehat{\rho^{\pi_e}_{s,a}} (s,a)}{\widehat{\rho^{\pi_e}_{s}} (s) +  \widehat{\rho^{\pi}_{s}} (s)}
    \label{adjust_error}
  \end{equation}
  
  Figure \ref{vbt} gives an illustration of the comparison between the reward functions in the form of Equation \eqref{mis_lead} and Equation \eqref{adjust_error}.
  In Figure \ref{haha}, we select two states where $\widehat{\pi_e(a | s_1)} = \widehat{\pi_e(a | s_1)} = \widehat{\pi_e(a | s)}$ (the blue dot line), $\widehat{\rho^{\pi_e}(s_1)} = \widehat{\rho^{\pi_e}(s_2)} = 1$,
  $\widehat{\rho^{\pi}(s_1)} = 0.1, \widehat{\rho^{\pi}(s_2)} = 0.05$.
  
\subsection{Algorithm}

\begin{algorithm}[tb]
  \caption{PDEIL:Probability Density Estimation based Imitation Learning }
  \label{algo} 
    \begin{tabular}{ll}
      \textbf{Input:} 
      & Expert demonstrations $D = \{s_i^e, a_i^e\}_{i \in [1:D]}$ \\ 
      & An agent's policy model $\pi$ \\
      & Probability Density Estimator $\widehat{\rho_{s,a}^{\pi_e}}$ \\
      & Probability Density Estimator $\widehat{\rho_{s}^{\pi_e}}$ \\
      & Probability Density Estimator $\widehat{\rho_{s}^{\pi}}$ \\
      &  A state’s buffer $R$ \\
    \textbf{Parameter}: &Number of epochs $N$ \\
    &Number of trying steps $T$ \\
    &Number of learning steps $L$ \\
    &Trade-off parameter $\alpha$ \\
    \textbf{Output}: & An agent with a desired policy \\
   \end{tabular}
   \
  \begin{algorithmic}[1] 
  %[1] enables line numbers
  \STATE Watching: train $\widehat{\rho_{s,a}^{\pi_e}}$ using the expert demonstrations $D$; train $\widehat{\rho_{s}^{\pi_e}}$ by using $\{s_i^e\}_{i \in [1:D]}$\\
  \FOR{\texttt{$i$ in $1:N$}}
          \STATE Trying: agent with model $\pi$ interacts with the environment for $T$ steps and save $s_1, s_2 \dots s_T$ into $R$
          \STATE Train  $\widehat{\rho_{s}^{\pi}}$ using $R$ and update $\widehat{\rho_{s}^{\pi}}$
          \STATE Clear $R$
          \STATE Update reward function using Equation \eqref{origin_adjust_error}
          \STATE Learning: agent performs learning for $L$ steps using an RL algorithm and updates its model $\pi$
  \ENDFOR
  \end{algorithmic}
  \end{algorithm}

  We present a novel IL algorithm based on the reward function in Equation \eqref{origin_adjust_error} and the 
  pseudo code of the algorithm is shown in Algorithm \ref{algo}.
  
  The framework of PDEIL consists of three major components, watching, trying and learning. In the watching part, 
  the agent  watches expert demonstrations, and uses these demonstrations to train 
  $\widehat{\rho_{s,a}^{\pi_e}}$ and $\widehat{\rho_{s}^{\pi_e}}$. In the trying part,
  the agent interacts with the environment for some steps to make better understanding of the environment as well as to 
  train and update $\widehat{\rho_{s}^{\pi}}$. In the learning part, the agent uses an RL algorithm to improve itself by further approximating the expert. 
  In a complete training process, our agent watches once and iteratively performs the trying and learning operations 
  several times as necessary, similar to the framework in \cite{zhou2019watch}.

  Probability density estimation is a classical problem that has been extensively investigated in statistics. 
  For a discrete random variable, the most common way to estimate its probability is using a frequency table representing
  its distribution. For a continuous random variable, there are two different ways to estimate its probability density: the parametric density estimation 
  and the nonparametric density estimation. 
  Parametric density estimation methods model the overall distribution as a certain distribution family, 
  such as Gaussian distributions, and the key is to determine the parameters of the distribution.  
  Nonparametric density estimation methods are model-free and
  one of the most commonly used methods is the kernel density estimation. % 这里需要介绍一些kde是怎么做的。
  For a mixed random variable, such as$ (s, a)$ where $s$ is continuous and $a$ is discrete, given the fact of $p (s,a) = p (a|s)p (s)$, 
  this joint distribution estimation problem can be transformed into a continuous distribution 
  estimation problem and a conditional probability estimation problem. For the conditional probability estimation, 
  we can regard it as a classification task, and, 
  for instance, use an SVM  \cite{noble2006support} classifier or other classifier methods to estimate $p (a|s)$.

  \section{Experiments}

  In the experimental studies, we aim to answer the following questions:
  
  \begin{enumerate}
    \item How does PDEIL perform in different environments?
    \item Is PDEIL more efficient than other IL algorithms?
    \item Is the recovered reward by PDEIL close to the ground truth reward?
    \item Does the misleading reward problem really exist?
  \end{enumerate}
  
  Our experiments were conducted in two environments in Gym, which is an open source platform for studying reinforcement learning algorithms. 
  We chose two classical control environments,
  CartPole and Pendulum (see Figure \ref{environments}), where CartPole is a discrete action space environment and Pendulum is a continuous action space 
  environment. To evaluate PDEIL, we hid the original reward signals of these two environments during training.  
  For CartPole, two Gaussian models were used to estimate $\rho^{\pi_e}_{s}$ and $\rho^{\pi}_{s}$ and an SVM model was used to estimate $\pi_e (a | s)$.
  PPO  \cite{schulman2017proximal} was used
  to update the policy in the learning steps. 
  For Pendulum, we used three Gaussian models to estimate $\rho^{\pi_e}_{s, a}$, $\rho^{\pi_e}_{s}$ and $\rho^{\pi}_{s}$, while 
  SAC  \cite{haarnoja2018soft}  was 
  used to update the policy in the the learning steps. Furthermore, 
  we set $N = 100$, $T = 1000$ in both environments with $L = 10$ and $5000$ on CartPole and
  Pendulum, respectively, according to some preliminary trials. % Our code is provided in \url{https://github.com/yliu15thu/deil}.%a github repo%

  \begin{figure}[tbh]
    \centering
    \subfigure[\small{CartPole}]{
      \includegraphics[width=.225\textwidth]{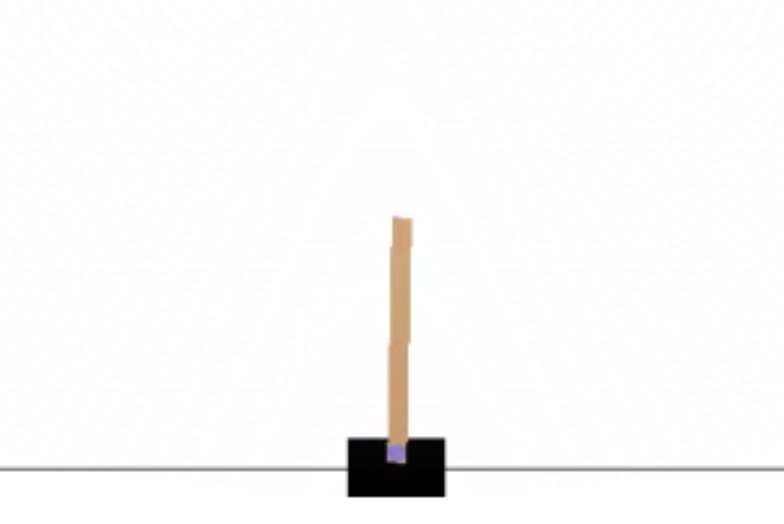}
    }
    \subfigure[\small{Pendulum}]{
      \includegraphics[width=.225\textwidth]{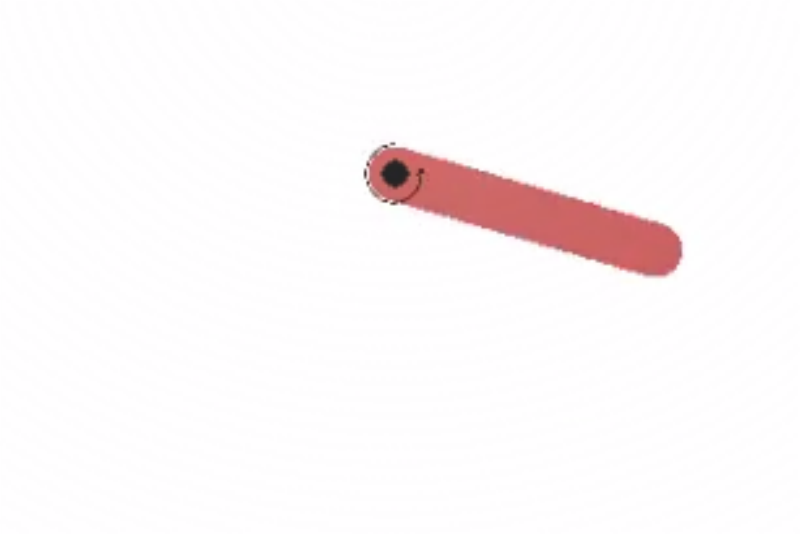}
    }
    \caption{The two experimental environments.}
    \label{environments}
  \end{figure}

  To answer Question 1, 
  we applied PDEIL with various expert demonstrations with $\alpha = 0.5$. 
  The performance of PDEIL in the two environments is shown in Figure \ref{af}. 
  The results clearly indicate that PDEIL can recover desired policies 
  that are reasonably close to the expert policies with a small amount of expert demonstrations in both discrete and continuous action spaces. 
  We also find that PDEIL may occasionally experience some slight stability issue. For example, the performance of PDEIL with 
  5 episodes of expert demonstrations on Pendulum (the blue line in Figure \ref{pe}) was a bit fluctuated. We argued for that there are two possible reasons: 
  \begin{enumerate*}[label=\roman*)]
    \item the estimated reward function is biased when the agent is learning;
    \item the optimization process for the neural network has some inherent instability.
  \end{enumerate*}
  %这里还可以多说一点点%

  \begin{figure}[tbh]
    \centering
    \subfigure[\small{CartPole-v1}]{
      \resizebox{0.225\textwidth}{!}{\includesvg{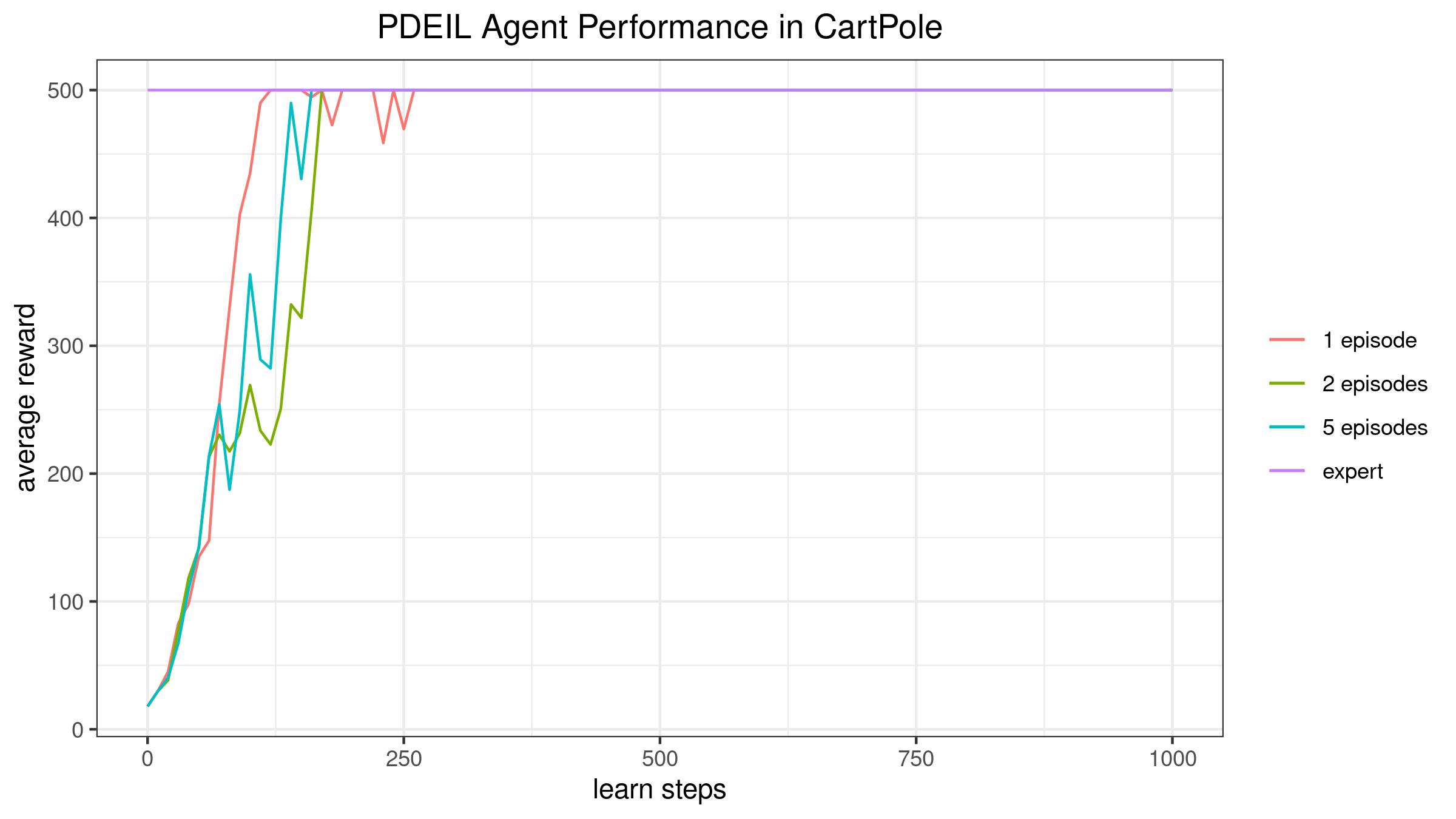}}
    }
    \subfigure[\small{Pendulum-v0}]{
      \resizebox{0.225\textwidth}{!}{\includesvg{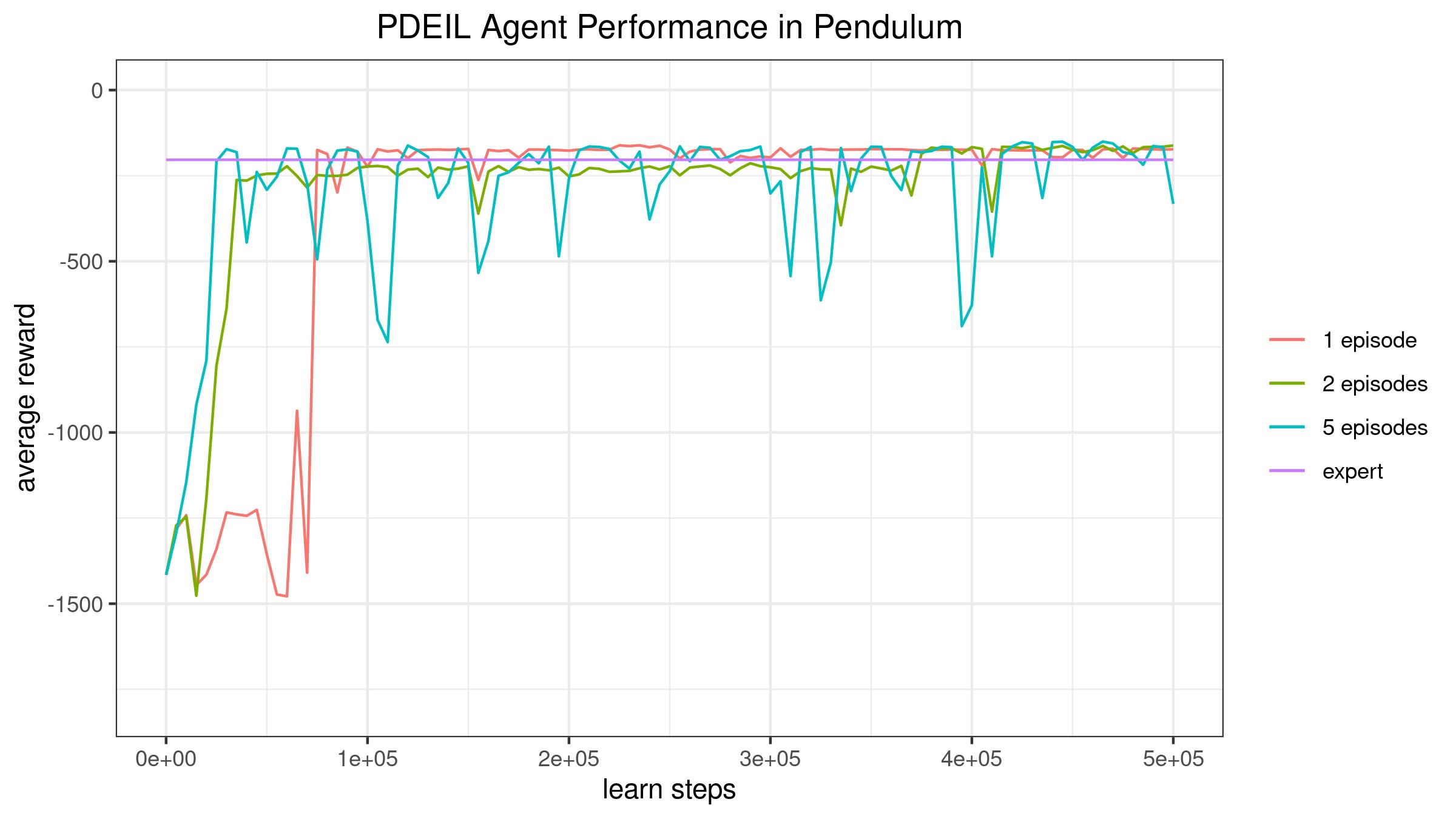}}
      \label{pe}
    }
    \caption{The performance of PDEIL with different episodes.}
    \label{af}
  \end{figure}

  To answer Question 2, we conducted extensive comparison among PDEIL, GAIL and BC algorithms with the same expert demonstrations. 
  The trade-off parameter $\alpha$ was also fixed to $0.5$ while the number of episodes of expert demonstrations was varied among 1, 2 and 5.
  In Figure \ref{epo}, it is obvious that PDEIL is much more efficient than GAIL and BC. 
  Although it seems that GAIL and PDEIL have a similar efficiency on CartPole, 
  PDEIL is more stable and GAIL requires many more interactions with the environment in each learning step. 
  Furthermore, PDEIL uses the Gaussian model  for reward estimation while GAIL uses a neural network, 
  which is more complicated and expensive in the training steps.

  \begin{figure*}[tbh]
    \centering
    \subfigure[\scriptsize{Performance on CartPole using 1 episode expert demonstrations}]{
      \resizebox{0.32\textwidth}{!}{\includesvg{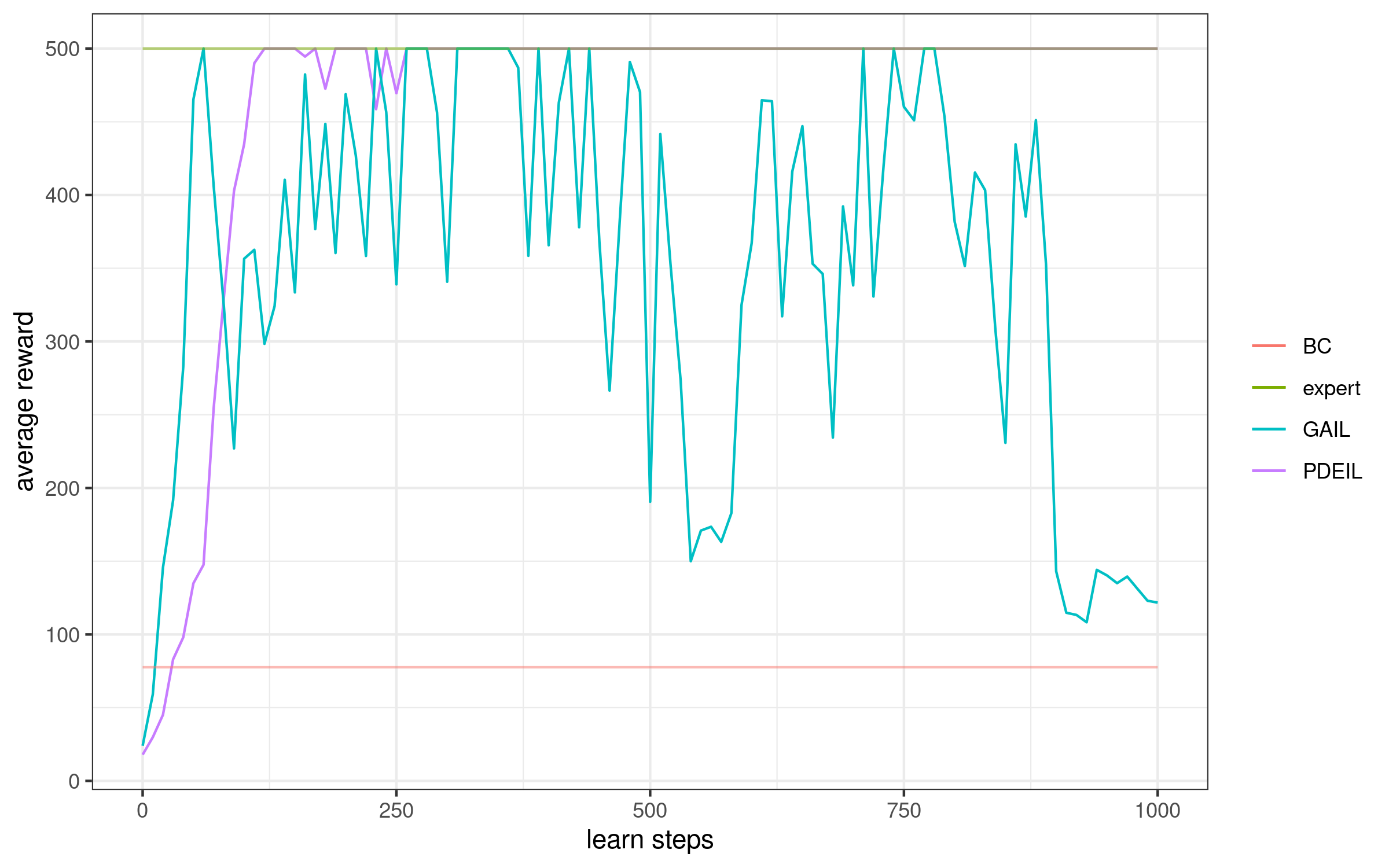}}
    }
    \subfigure[\scriptsize{Performance on CartPole using 2 episodes expert demonstrations}]{
      \resizebox{0.32\textwidth}{!}{\includesvg{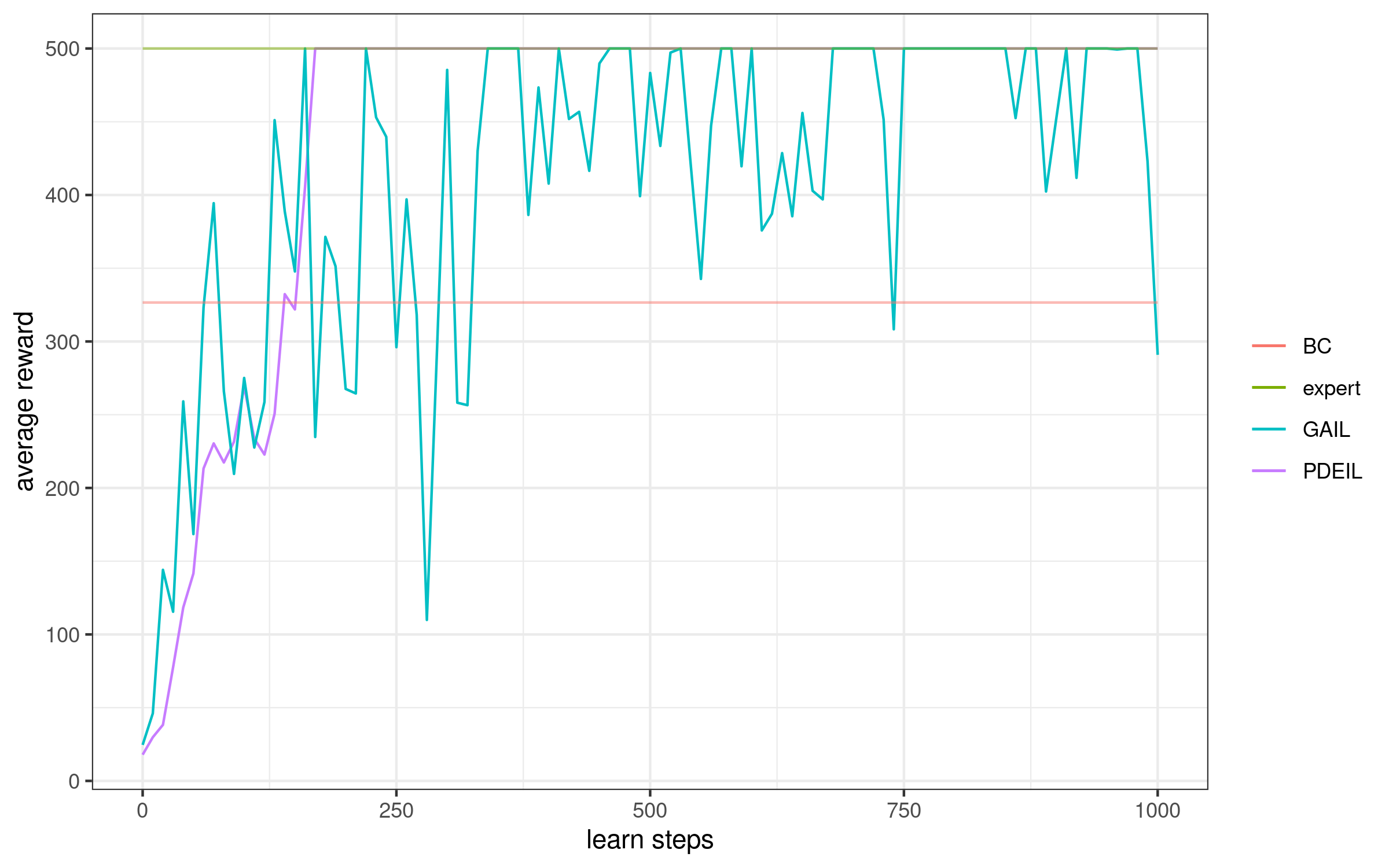}}
    }
    \subfigure[\scriptsize{Performance on CartPole using 5 episodes expert demonstrations}]{
      \resizebox{0.32\textwidth}{!}{\includesvg{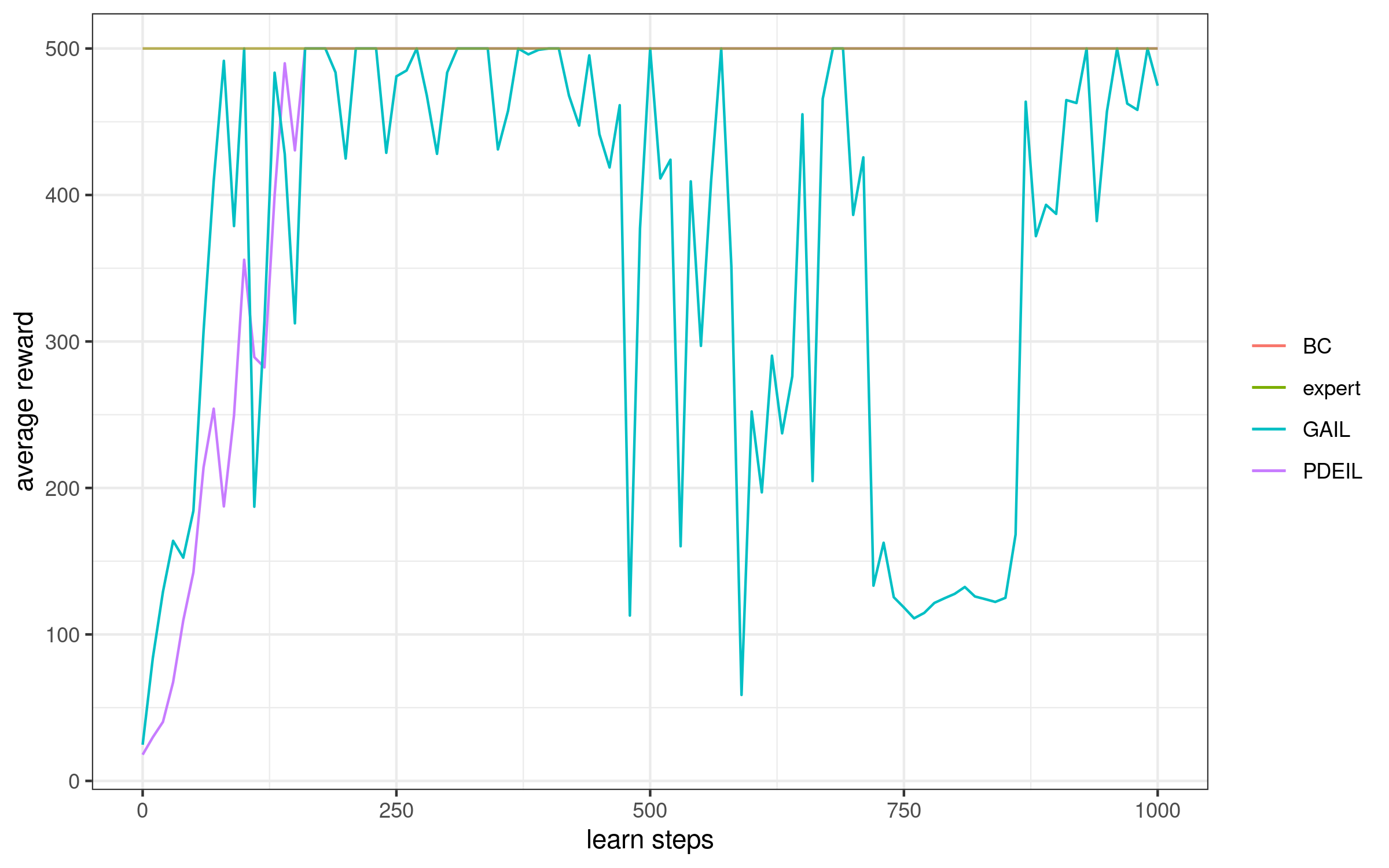}}
    }
    
    \subfigure[\scriptsize{Performance on Pendulum using 1 episode expert demonstrations}]{
      \resizebox{0.32\textwidth}{!}{\includesvg{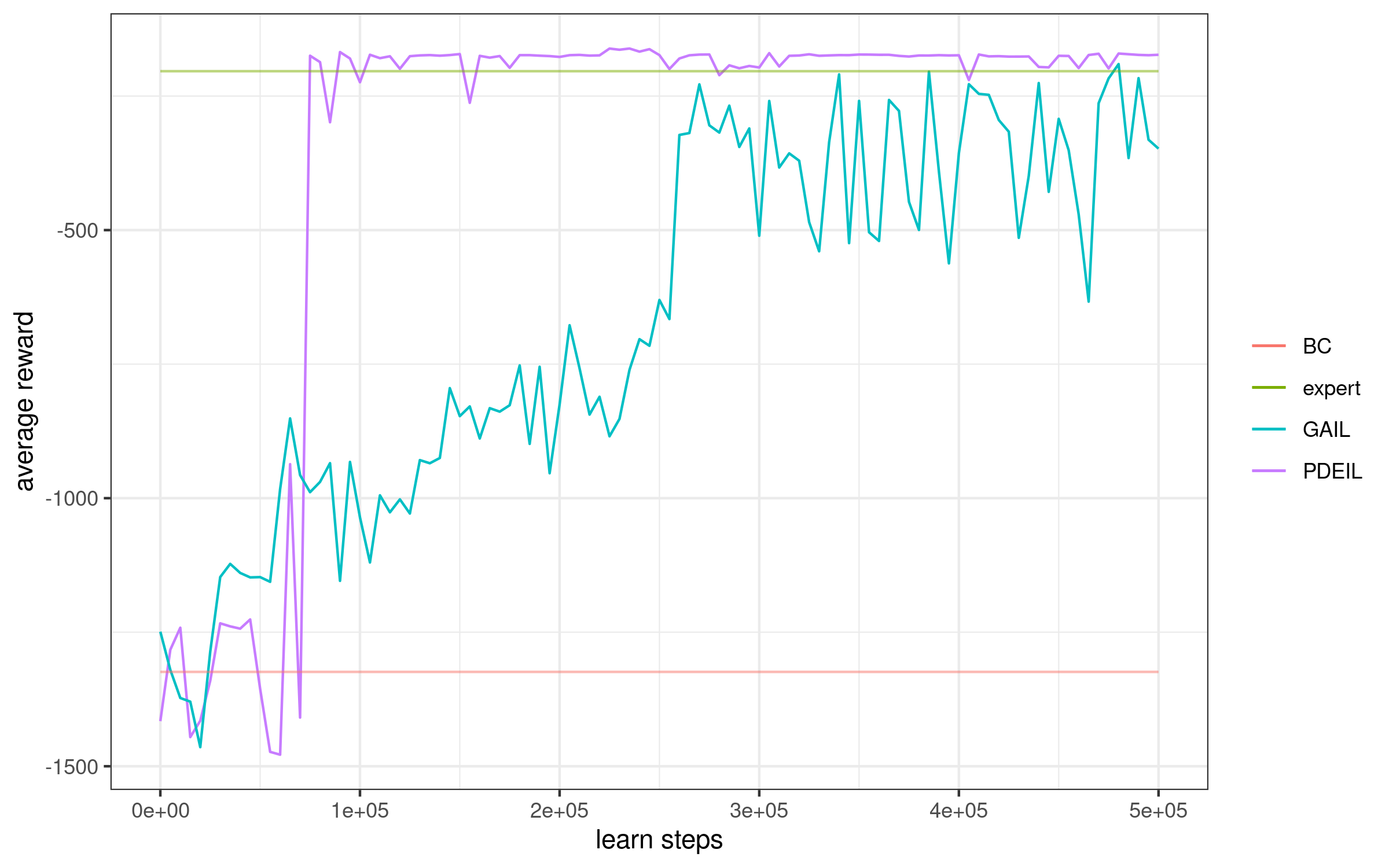}}
    }
    \subfigure[\scriptsize{Performance on Pendulum using 2 episodes expert demonstrations}]{
      \resizebox{0.32\textwidth}{!}{\includesvg{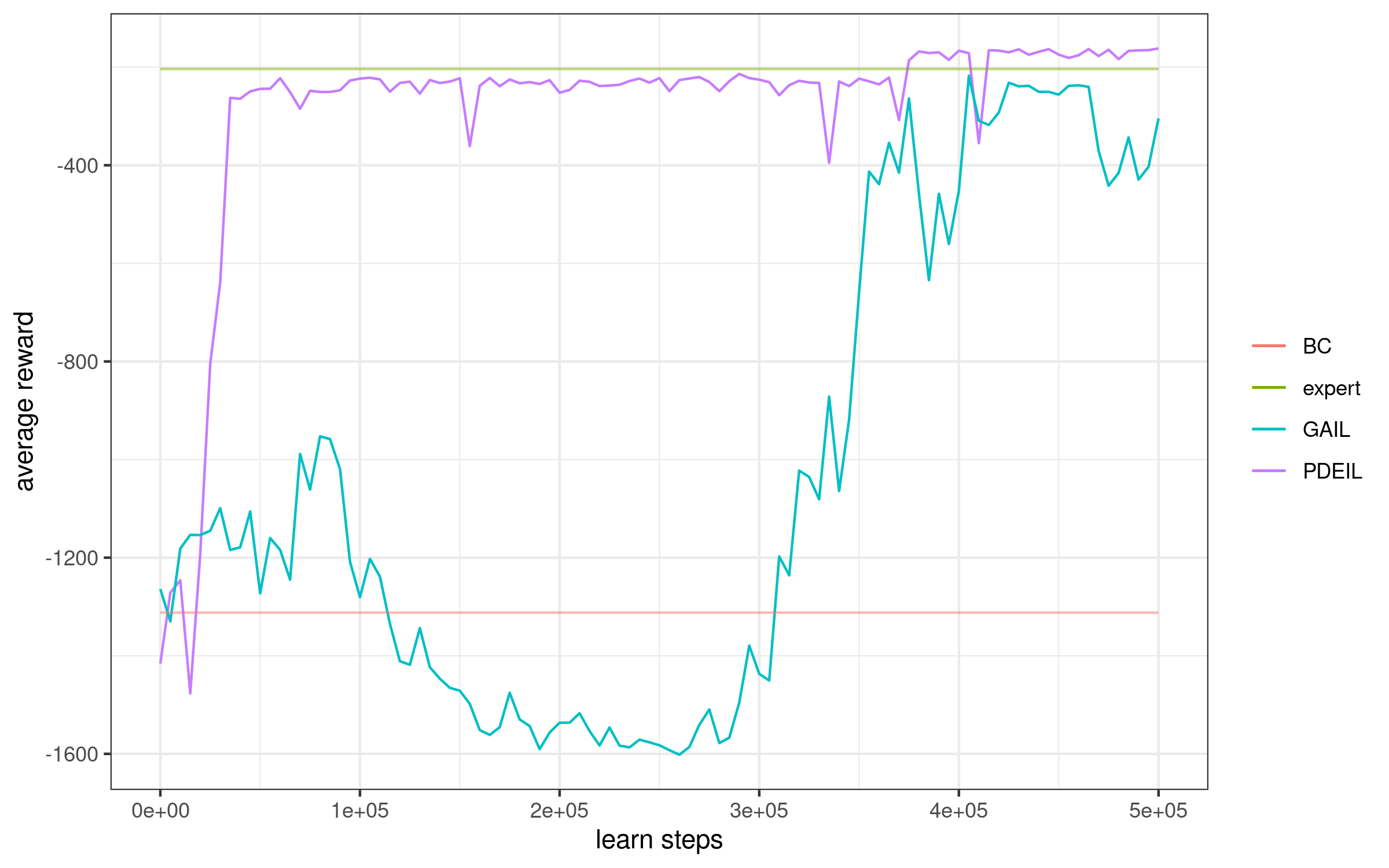}}
    }
    \subfigure[\scriptsize{Performance on Pendulum using 5 episodes expert demonstrations}]{
      \resizebox{0.32\textwidth}{!}{\includesvg{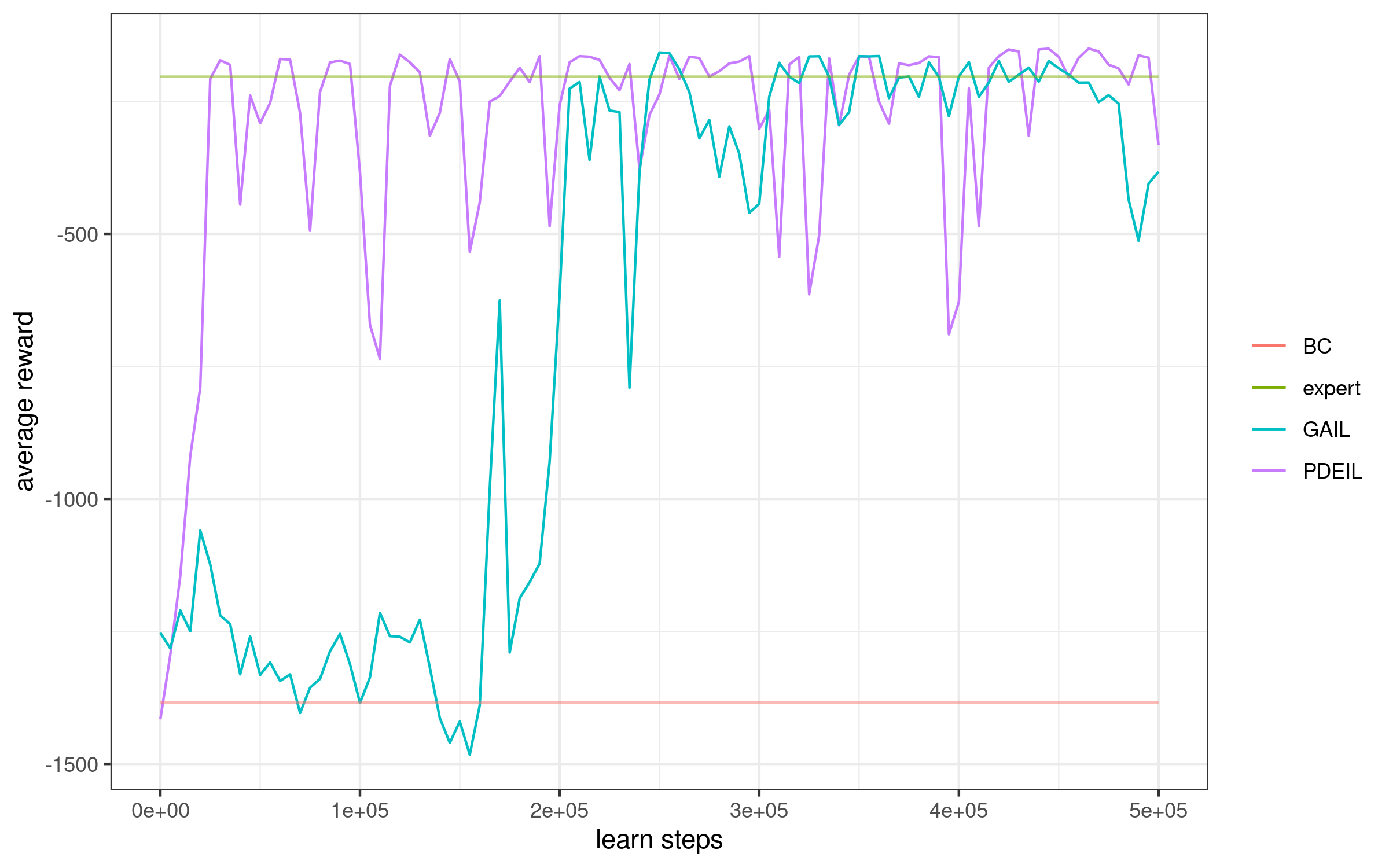}}
    }
    \caption{PDEIL vs GAIL and BC on CartPole and Pendulum environments.}
    \label{epo}
  \end{figure*}

To answer Question 3, we collected the recovered rewards and the ground truth rewards from the original environment in the trying steps in PDEIL. 
We chose the Pendulum environment as the example for the sake of visual demonstration, as its ground truth reward is continuous.
Each round of the experiment had 100 epochs from which the 13th to 16th epochs were  picked as examples for  illustration in Figure \ref{epo2}.
It is clear that the correlation between the recovered reward and the ground truth reward gets increasingly stronger,
which means that our proposed reward function can be expected to guide the RL algorithm to learn a competent policy.
\begin{figure}[tb]
  \centering
  \subfigure[\scriptsize{Rewards in the 13th epoch}]{
    \resizebox{0.2\textwidth}{!}{\includesvg{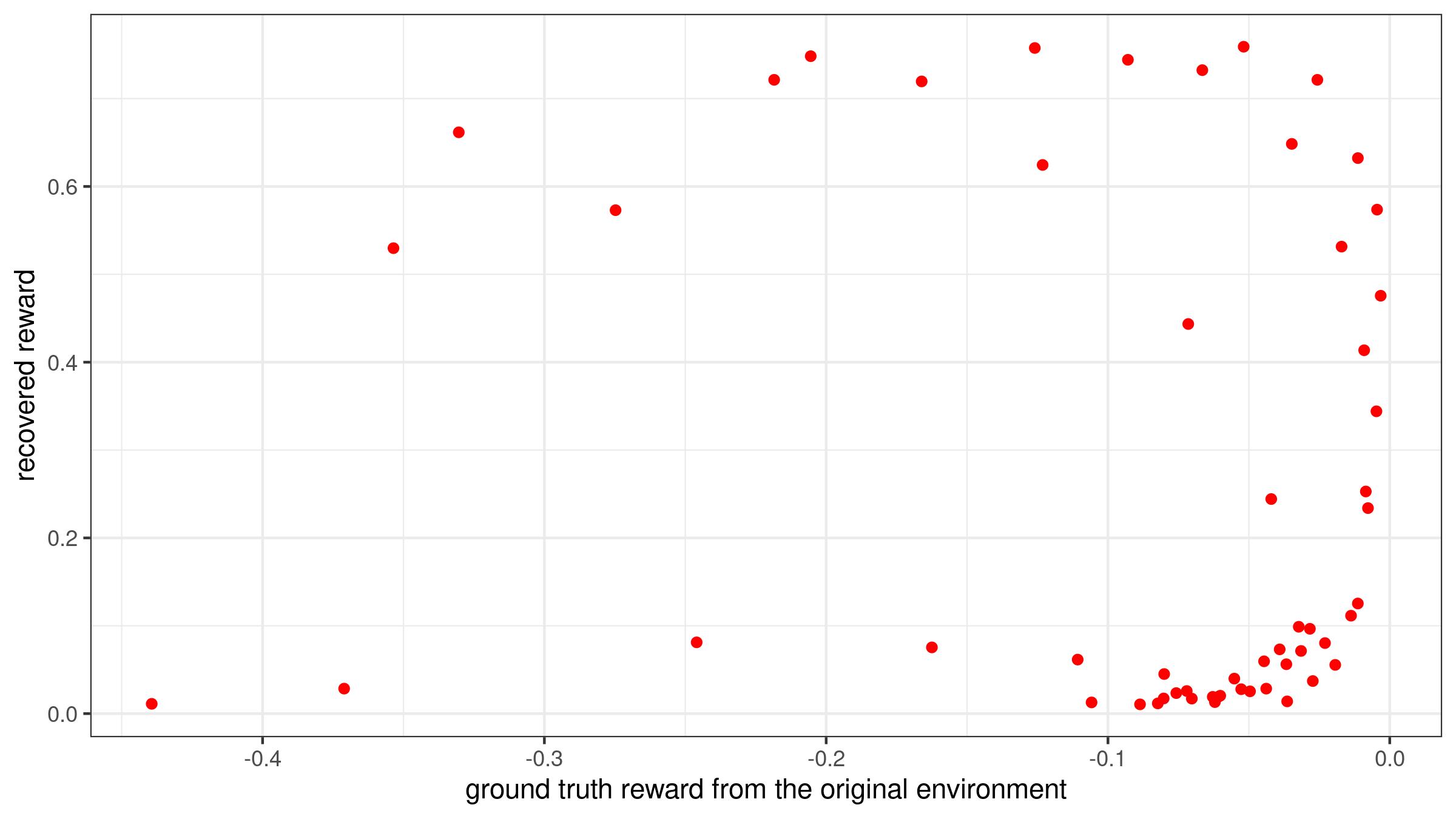}}
  }
  \subfigure[\scriptsize{Rewards in the 14th epoch}]{
    \resizebox{0.2\textwidth}{!}{\includesvg{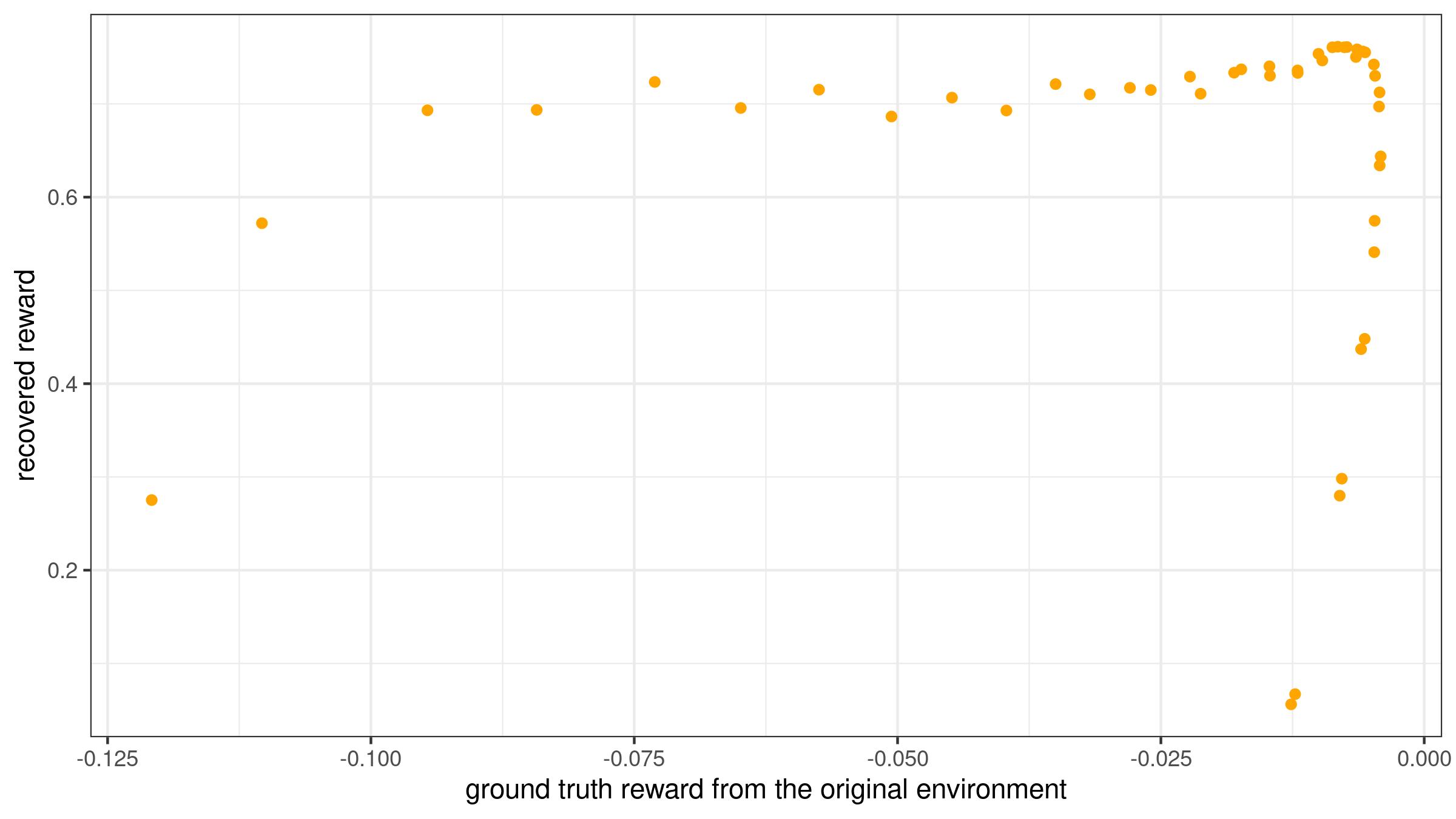}}
  }

  \subfigure[\scriptsize{Rewards in the 15th epoch}]{
    \resizebox{0.2\textwidth}{!}{\includesvg{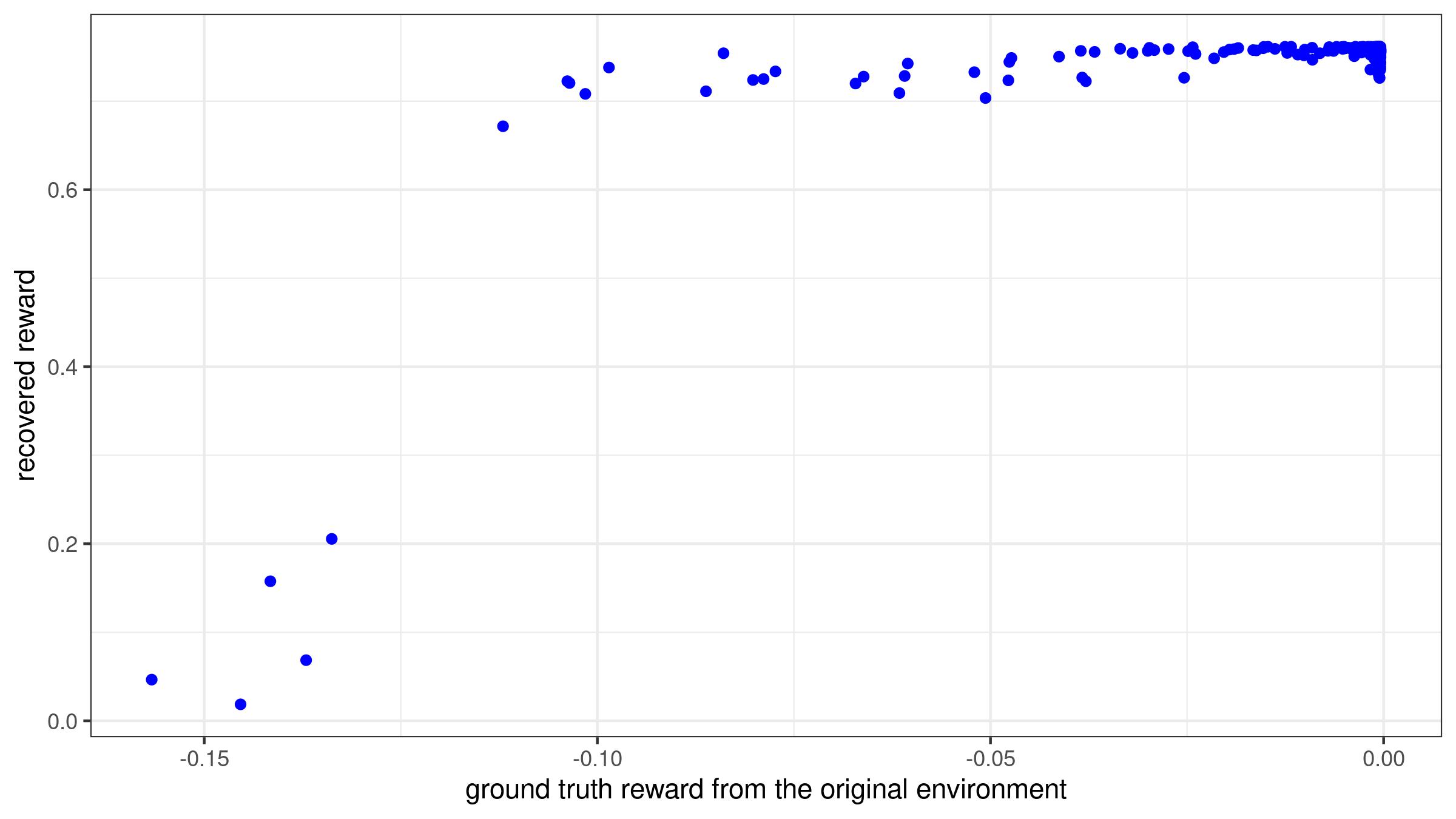}}
  }
  \subfigure[\scriptsize{Rewards in the 16th epoch}]{
    \resizebox{0.2\textwidth}{!}{\includesvg{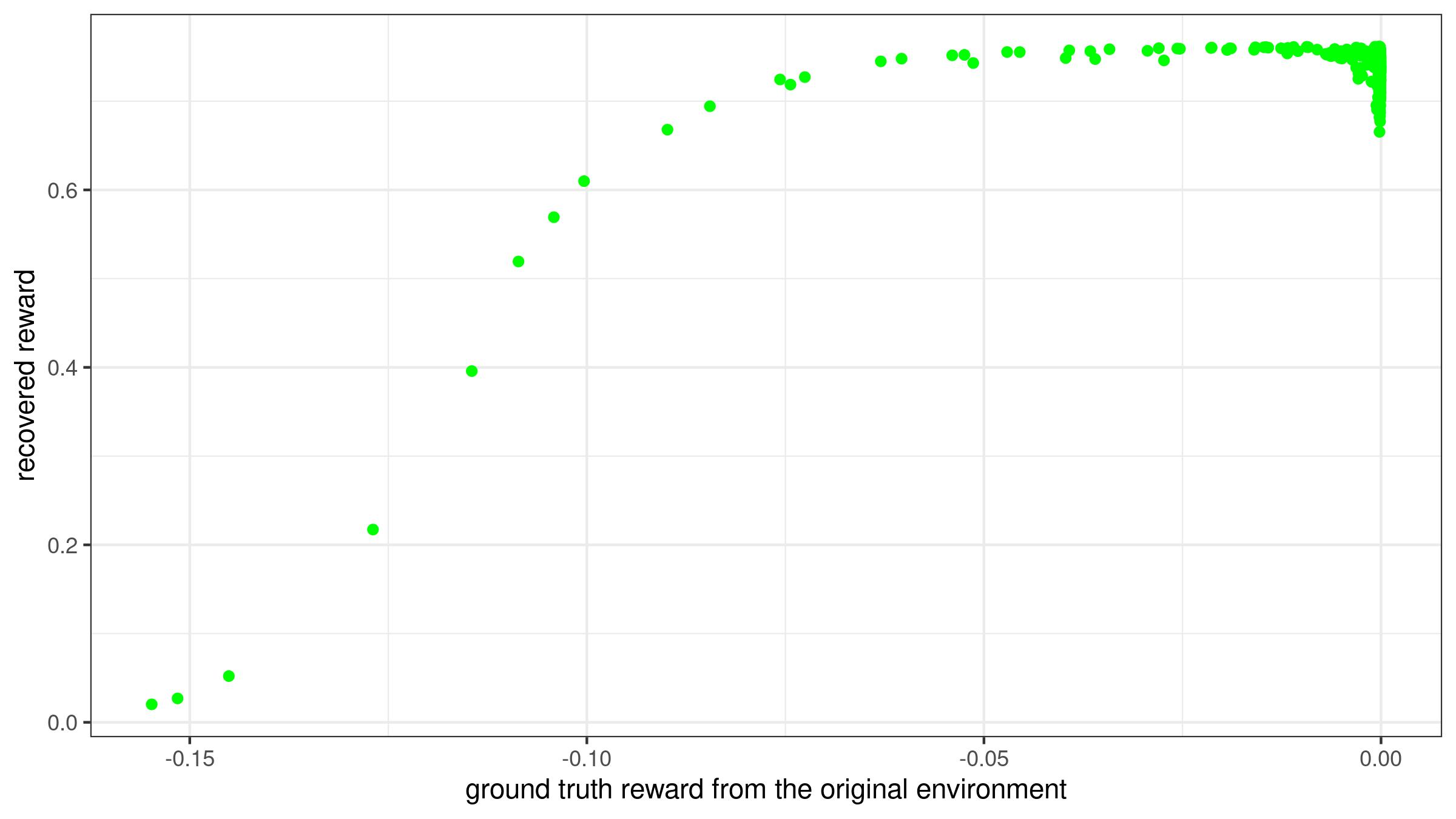}}
  }
  \caption{The correlation coefficients between the two rewards are 0.19 in the 13th learning epoch, 
  0.37 in the 14th learning epoch, 0.70 in the 15th learning epoch and 0.76 in the 16th learning epoch.}
  \label{epo2}
\end{figure}

\begin{figure}[tbh]
  \centering
  \subfigure[\small{CartPole-v1}]{
    \resizebox{0.225\textwidth}{!}{\includesvg{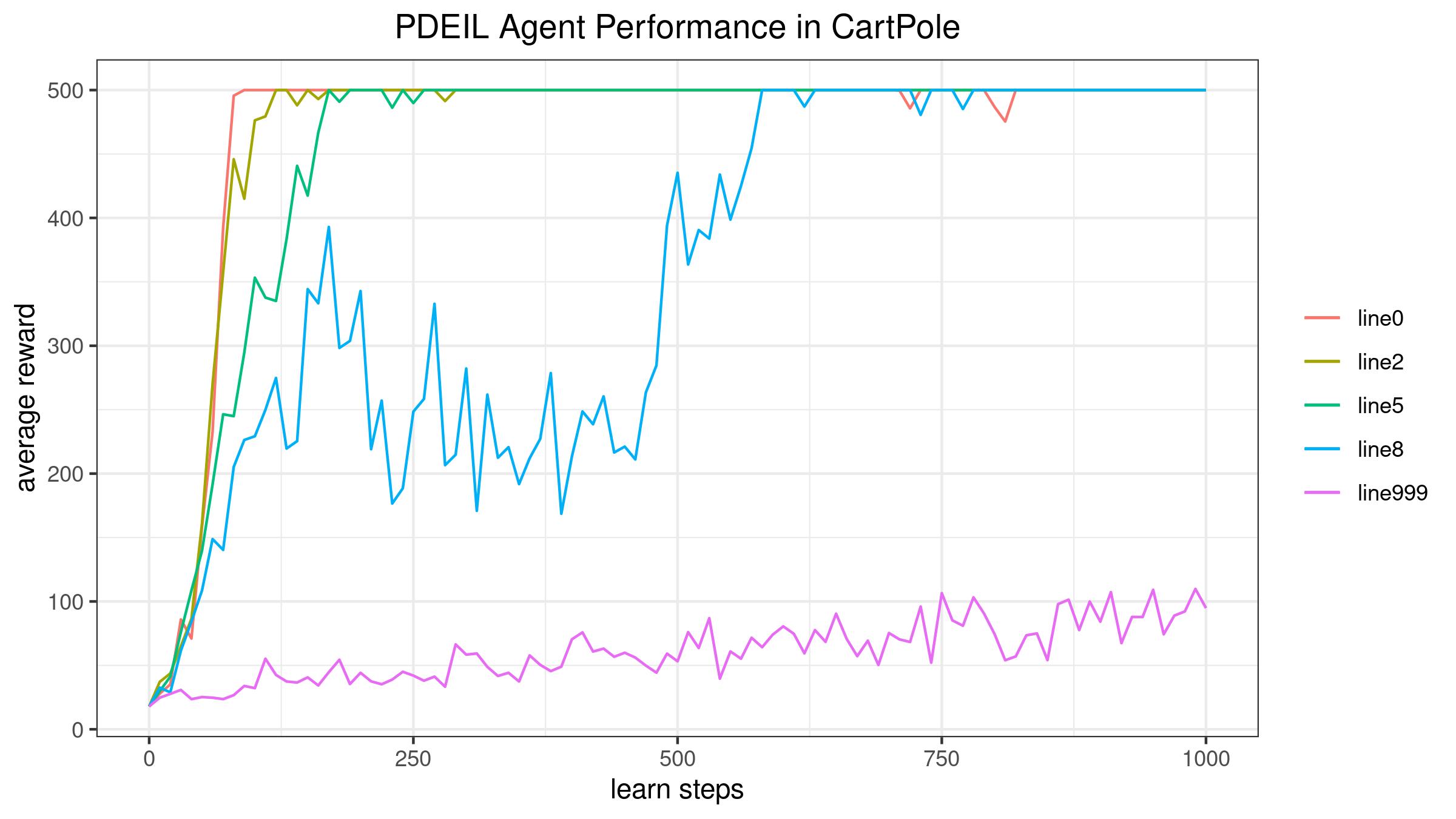}}
  }
  \subfigure[\small{Pendulum-v0}]{
    \resizebox{0.225\textwidth}{!}{\includesvg{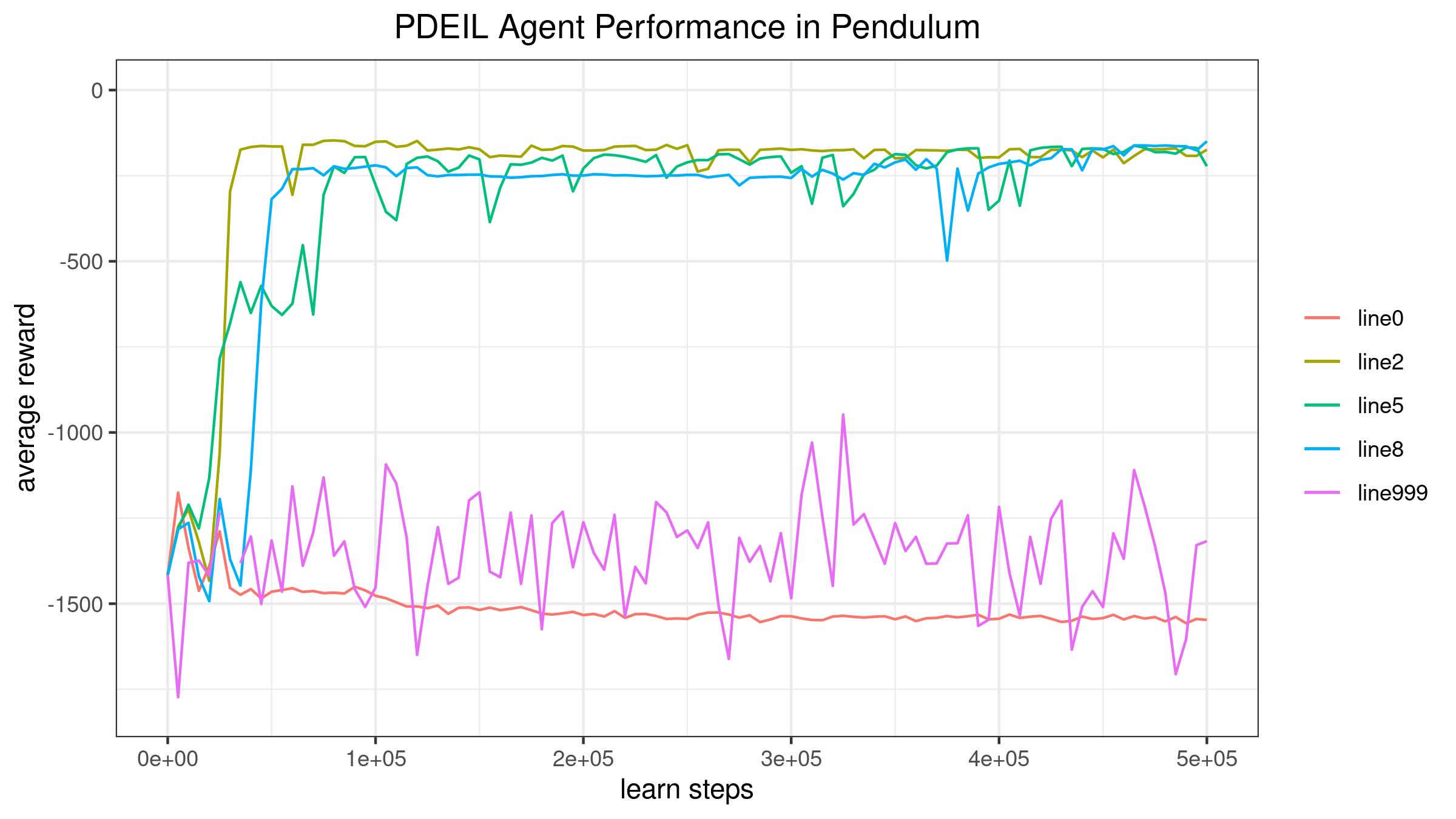}}
    \label{ee}
  }
  \caption{The average performance of PDEIL with various $\alpha$ values.}
  \label{sessad}
\end{figure}

To answer Question 4, in addition to the theoretical analysis in Section \ref{mr}, we conducted further empirical study as shown in Figure \ref{sessad}.
The introduction of the trade-off parameter $\alpha$ is to control the variance of estimated reward.
When $\alpha = 0$, we use the original reward function in Equation \ref{mis_lead}, and the misleading reward problem may occur. For example,
the agent using the original reward function  in the Pendulum environment(red line in Figure \ref{ee}) had a poor performance, 
which implied that the misleading reward problem occurred. 

\section{Conclusion and Discussion}

In this work, we proposed a brand-new reward function in the scenario of IRL, which has a concise and indicative form. 
We also proposed an algorithm called PDEIL based on this reward function, featuring a "watch-try-learn" style. 
In PDEIL, to recover the reward function, the agent watches the expert demonstrations 
and performs interactions with the environment, and uses RL algorithms to update the policy.

It is expected that our work may reveal a new perspective for IRL 
by transforming the original IRL problem into a density estimation problem. 
We can prove that, with a perfect probability density estimator, the corresponding optimal policy is identical to the expert policy 
as long as it is deterministic. 
However, constructing a good probability density estimator can be challenging in some cases, 
for example, when the state is of high dimensionality (e.g., an image). 
Consequently, further enhancing the efficacy of PDEIL with 
more competent probability density estimators will be a key direction for our future work.

\bibliographystyle{named}
\bibliography{ijcai21}

\begin{thebibliography}{}

\bibitem[\protect\citeauthoryear{Abbeel and
  Ng}{2004}]{abbeel2004apprenticeship}
Pieter Abbeel and Andrew~Y Ng.
\newblock Apprenticeship learning via inverse reinforcement learning.
\newblock In {\em Proceedings of the twenty-first international conference on
  Machine learning}, page~1, 2004.

\bibitem[\protect\citeauthoryear{Andrychowicz \bgroup \em et al.\egroup
  }{2020}]{andrychowicz2020learning}
OpenAI:~Marcin Andrychowicz, Bowen Baker, Maciek Chociej, Rafal Jozefowicz, Bob
  McGrew, Jakub Pachocki, Arthur Petron, Matthias Plappert, Glenn Powell, Alex
  Ray, et~al.
\newblock Learning dexterous in-hand manipulation.
\newblock {\em The International Journal of Robotics Research}, 39(1):3--20,
  2020.

\bibitem[\protect\citeauthoryear{Bagnell \bgroup \em et al.\egroup
  }{2007}]{bagnell2007boosting}
JA~Bagnell, Joel Chestnutt, David~M Bradley, and Nathan~D Ratliff.
\newblock Boosting structured prediction for imitation learning.
\newblock In {\em Advances in Neural Information Processing Systems}, pages
  1153--1160, 2007.

\bibitem[\protect\citeauthoryear{Brantley \bgroup \em et al.\egroup
  }{2019}]{brantley2019disagreement}
Kiant{\'e} Brantley, Wen Sun, and Mikael Henaff.
\newblock Disagreement-regularized imitation learning.
\newblock In {\em International Conference on Learning Representations}, 2019.

\bibitem[\protect\citeauthoryear{Dadashi \bgroup \em et al.\egroup
  }{2020}]{dadashi2020primal}
Robert Dadashi, L{\'e}onard Hussenot, Matthieu Geist, and Olivier Pietquin.
\newblock Primal wasserstein imitation learning.
\newblock {\em arXiv preprint arXiv:2006.04678}, 2020.

\bibitem[\protect\citeauthoryear{Goodfellow \bgroup \em et al.\egroup
  }{2014}]{goodfellow2014generative}
Ian Goodfellow, Jean Pouget-Abadie, Mehdi Mirza, Bing Xu, David Warde-Farley,
  Sherjil Ozair, Aaron Courville, and Yoshua Bengio.
\newblock Generative adversarial nets.
\newblock {\em Advances in neural information processing systems},
  27:2672--2680, 2014.

\bibitem[\protect\citeauthoryear{Haarnoja \bgroup \em et al.\egroup
  }{2018}]{haarnoja2018soft}
Tuomas Haarnoja, Aurick Zhou, Pieter Abbeel, and Sergey Levine.
\newblock Soft actor-critic: Off-policy maximum entropy deep reinforcement
  learning with a stochastic actor.
\newblock {\em arXiv preprint arXiv:1801.01290}, 2018.

\bibitem[\protect\citeauthoryear{Ho and Ermon}{2016}]{ho2016generative}
Jonathan Ho and Stefano Ermon.
\newblock Generative adversarial imitation learning.
\newblock {\em arXiv preprint arXiv:1606.03476}, 2016.

\bibitem[\protect\citeauthoryear{Lillicrap \bgroup \em et al.\egroup
  }{2015}]{lillicrap2015continuous}
Timothy~P Lillicrap, Jonathan~J Hunt, Alexander Pritzel, Nicolas Heess, Tom
  Erez, Yuval Tassa, David Silver, and Daan Wierstra.
\newblock Continuous control with deep reinforcement learning.
\newblock {\em arXiv preprint arXiv:1509.02971}, 2015.

\bibitem[\protect\citeauthoryear{Mnih \bgroup \em et al.\egroup
  }{2013}]{mnih2013playing}
Volodymyr Mnih, Koray Kavukcuoglu, David Silver, Alex Graves, Ioannis
  Antonoglou, Daan Wierstra, and Martin Riedmiller.
\newblock Playing atari with deep reinforcement learning.
\newblock {\em arXiv preprint arXiv:1312.5602}, 2013.

\bibitem[\protect\citeauthoryear{Nagabandi \bgroup \em et al.\egroup
  }{2018}]{nagabandi2018neural}
Anusha Nagabandi, Gregory Kahn, Ronald~S Fearing, and Sergey Levine.
\newblock Neural network dynamics for model-based deep reinforcement learning
  with model-free fine-tuning.
\newblock In {\em 2018 IEEE International Conference on Robotics and Automation
  (ICRA)}, pages 7559--7566. IEEE, 2018.

\bibitem[\protect\citeauthoryear{Ng \bgroup \em et al.\egroup
  }{2000}]{ng2000algorithms}
Andrew~Y Ng, Stuart~J Russell, et~al.
\newblock Algorithms for inverse reinforcement learning.
\newblock In {\em Icml}, volume~1, page~2, 2000.

\bibitem[\protect\citeauthoryear{Noble}{2006}]{noble2006support}
William~S Noble.
\newblock What is a support vector machine?
\newblock {\em Nature biotechnology}, 24(12):1565--1567, 2006.

\bibitem[\protect\citeauthoryear{Pomerleau}{1991}]{pomerleau1991efficient}
Dean~A Pomerleau.
\newblock Efficient training of artificial neural networks for autonomous
  navigation.
\newblock {\em Neural computation}, 3(1):88--97, 1991.

\bibitem[\protect\citeauthoryear{Rajeswaran \bgroup \em et al.\egroup
  }{2017}]{rajeswaran2017learning}
Aravind Rajeswaran, Vikash Kumar, Abhishek Gupta, Giulia Vezzani, John
  Schulman, Emanuel Todorov, and Sergey Levine.
\newblock Learning complex dexterous manipulation with deep reinforcement
  learning and demonstrations.
\newblock {\em arXiv preprint arXiv:1709.10087}, 2017.

\bibitem[\protect\citeauthoryear{Reddy \bgroup \em et al.\egroup
  }{2019}]{reddy2019sqil}
Siddharth Reddy, Anca~D Dragan, and Sergey Levine.
\newblock Sqil: Imitation learning via reinforcement learning with sparse
  rewards.
\newblock {\em arXiv preprint arXiv:1905.11108}, 2019.

\bibitem[\protect\citeauthoryear{Ross and Bagnell}{2010}]{ross2010efficient}
St{\'e}phane Ross and Drew Bagnell.
\newblock Efficient reductions for imitation learning.
\newblock In {\em Proceedings of the thirteenth international conference on
  artificial intelligence and statistics}, pages 661--668, 2010.

\bibitem[\protect\citeauthoryear{Russell}{1998}]{russell1998learning}
Stuart Russell.
\newblock Learning agents for uncertain environments.
\newblock In {\em Proceedings of the eleventh annual conference on
  Computational learning theory}, pages 101--103, 1998.

\bibitem[\protect\citeauthoryear{Schulman \bgroup \em et al.\egroup
  }{2017}]{schulman2017proximal}
John Schulman, Filip Wolski, Prafulla Dhariwal, Alec Radford, and Oleg Klimov.
\newblock Proximal policy optimization algorithms.
\newblock {\em arXiv preprint arXiv:1707.06347}, 2017.

\bibitem[\protect\citeauthoryear{Silver \bgroup \em et al.\egroup
  }{2014}]{silver2014deterministic}
David Silver, Guy Lever, Nicolas Heess, Thomas Degris, Daan Wierstra, and
  Martin Riedmiller.
\newblock Deterministic policy gradient algorithms.
\newblock 2014.

\bibitem[\protect\citeauthoryear{Silver \bgroup \em et al.\egroup
  }{2017}]{silver2017mastering}
David Silver, Julian Schrittwieser, Karen Simonyan, Ioannis Antonoglou, Aja
  Huang, Arthur Guez, Thomas Hubert, Lucas Baker, Matthew Lai, Adrian Bolton,
  et~al.
\newblock Mastering the game of go without human knowledge.
\newblock {\em nature}, 550(7676):354--359, 2017.

\bibitem[\protect\citeauthoryear{Sutton \bgroup \em et al.\egroup
  }{1998}]{sutton1998introduction}
Richard~S Sutton, Andrew~G Barto, et~al.
\newblock {\em Introduction to reinforcement learning}, volume 135.
\newblock MIT press Cambridge, 1998.

\bibitem[\protect\citeauthoryear{Wang \bgroup \em et al.\egroup
  }{2019}]{wang2019random}
Ruohan Wang, Carlo Ciliberto, Pierluigi Amadori, and Yiannis Demiris.
\newblock Random expert distillation: Imitation learning via expert policy
  support estimation.
\newblock {\em arXiv preprint arXiv:1905.06750}, 2019.

\bibitem[\protect\citeauthoryear{Zhou \bgroup \em et al.\egroup
  }{2019}]{zhou2019watch}
Allan Zhou, Eric Jang, Daniel Kappler, Alex Herzog, Mohi Khansari, Paul
  Wohlhart, Yunfei Bai, Mrinal Kalakrishnan, Sergey Levine, and Chelsea Finn.
\newblock Watch, try, learn: Meta-learning from demonstrations and reward.
\newblock {\em arXiv preprint arXiv:1906.03352}, 2019.

\bibitem[\protect\citeauthoryear{Ziebart \bgroup \em et al.\egroup
  }{2008}]{ziebart2008maximum}
Brian~D Ziebart, Andrew~L Maas, J~Andrew Bagnell, and Anind~K Dey.
\newblock Maximum entropy inverse reinforcement learning.
\newblock In {\em Aaai}, volume~8, pages 1433--1438. Chicago, IL, USA, 2008.

\end{thebibliography}

\end{document}